\documentclass{article}

\usepackage[preprint,nonatbib]{neurips_2019}

\usepackage[utf8]{inputenc} 
\usepackage[T1]{fontenc}    
\usepackage{hyperref}       
\usepackage{url}            
\usepackage{booktabs}       
\usepackage{amsfonts}       
\usepackage{nicefrac}       
\usepackage{microtype}      

\usepackage{amsmath,amssymb,amsthm}

\usepackage{float}

\usepackage{graphicx}
\usepackage{subcaption}
\usepackage{algorithm}
\usepackage{algorithmic}
\usepackage{enumerate}

\newtheorem{theorem}{Theorem}
\newtheorem{prop}{Proposition}

\newtheorem{lemma}{Lemma}
\newtheorem{cor}{Corollary}

\newtheorem{asm}{Assumption}

\DeclareMathOperator*{\sgn}{sign}
\DeclareMathOperator*{\conv}{conv}
\DeclareMathOperator*{\argmin}{argmin}
\DeclareMathOperator*{\argmax}{argmax}
\DeclareMathOperator*{\rob}{rob}
\DeclareMathOperator*{\poly}{poly}
\DeclareMathOperator*{\margin}{margin}

\newcommand{\PP}{\mathbb{P}}
\newcommand{\EE}{\mathbb{E}}

\newcommand{\real}{\mathbb{R}}

\newcommand{\mA}{\mathcal{A}}

\newcommand{\mC}{\mathcal{C}}

\newcommand{\mO}{\mathcal{O}}
\newcommand{\tmO}{\tilde{\mathcal{O}}}
\newcommand{\mS}{\mathcal{S}}

\newcommand{\mD}{\mathcal{D}}
\newcommand{\mX}{\mathcal{X}}
\newcommand{\mY}{\mathcal{Y}}
\newcommand{\mW}{\mathcal{W}}

\newcommand{\ie}{\emph{i.e.}}

\newcommand{\z}[1]{z^{(#1)}}

\title{Convergence and Margin of Adversarial Training\\ on Separable Data}

\author{%
  Zachary Charles\\
  University of Wisconsin-Madison\\
  \texttt{zcharles@wisc.edu} \\
  \And
  Shashank Rajput \\
  University of Wisconsin-Madison \\
  \texttt{rajput3@wisc.edu} \\
  \AND
  Stephen Wright \\
  University of Wisconsin-Madison \\
  \texttt{swright@cs.wisc.edu} \\
  \And
  Dimitris Papailiopoulos \\
  University of Wisconsin-Madison\\
  \texttt{dimitris@papail.io} \\
}

\begin{document}

\maketitle

\begin{abstract}
Adversarial training is a technique for training robust machine learning models.
To encourage robustness, it iteratively computes adversarial examples for the model, and then re-trains on these examples via some update rule. 
This work analyzes the performance of adversarial training on linearly separable data, and provides bounds on the number of iterations required for large margin.  
We show that when the update rule is given by an arbitrary empirical risk minimizer, adversarial training may require exponentially many iterations to obtain large margin. 
However, if gradient or stochastic gradient update rules are used, only polynomially many iterations are required to find a large-margin separator. 
By contrast, without the use of adversarial examples, gradient methods may require exponentially many iterations to achieve large margin.  
Our results are derived by showing that adversarial training with gradient updates minimizes a robust version of the empirical risk at a $\mathcal{O}(\ln(t)^2/t)$ rate, despite non-smoothness.
We corroborate our theory empirically.
\end{abstract}

\section{Introduction}

Machine learning models trained through standard methods
often lack robustness against {\it adversarial examples}. These are small
perturbations of input examples, designed to ``fool'' the model into
misclassifying the original input \cite{biggio2013evasion, goodfellow2014explaining, nguyen2015deep, szegedy2013intriguing}. Unfortunately, even small perturbations can cause a large degradation in the test accuracy of popular machine learning models, including deep
neural networks \cite{szegedy2013intriguing}. 
This lack of robustness has spurred a large body of work on designing attack methods for
crafting effective adversarial examples
\cite{grosse2016adversarial,hendrik2017universal,moosavi2016deepfool,mopuri2017fast,
  papernot2016transferability, tramer2017ensemble} and defense mechanisms
for training models that are more robust to norm bounded perturbations \cite{tramer2017ensemble, madry2017towards, sinha2017certifiable, zantedeschi2017efficient, samangouei2018defense, ilyas2017robust, shaham2018understanding}.

{\it Adversarial training} is a family of optimization-based methods for defending
against adversarial perturbations. 
These methods generally operate by computing adversarial examples, and retraining the model on these examples \cite{goodfellow2014explaining, madry2017towards, shaham2018understanding}. This two-step process is repeated iteratively. While adversarial training methods have achieved empirical success
\cite{madry2017towards, shaham2018understanding,ford2019adversarial, hendrycks2018benchmarking}, there is currently little theoretical analysis of their convergence and capacity for guaranteeing robustness.

A parallel line of research has investigated whether standard
optimization methods, such as gradient descent (GD) and stochastic
gradient descent (SGD), exhibit an {\it implicit bias} toward
robust and generalizable models \cite{gunasekar2018characterizing, gunasekar2018implicit, ji2018risk, nacson2019convergence, nacson2018stochastic, soudry2018implicit}. 
This line of work shows
that GD and SGD both converge to the max-margin classifier of linearly
separable data, provided that the loss function is chosen
appropriately. Notably, the max-margin classifier is the most robust
 model against $\ell_2$ bounded perturbations. Thus, gradient descent is indeed biased towards robustness in some settings. Unfortunately, convergence to this desirable limit can be slow, and in some cases an exponential number of iterations may be needed \cite{nacson2019convergence, nacson2018stochastic, soudry2018implicit}.

\paragraph{Our contributions.}

In this work, we merge these two previously separate lines of work,
studying whether (and how) various types of adversarial training
exhibit a bias towards robust models. We focus on linear
classification tasks and study robustness primarily through the lens of margin, the
minimum distance between the classification boundary and the
(unperturbed) training examples. Our results show that alone, neither
adversarial training with generic update rules, nor gradient-based training on the original data set, can find large-margin models quickly. 
However, by combining the two --- interspersing
gradient-based update rules with the addition of adversarial examples to the
training set --- we can train robust models significantly faster.

We show that for logistic regression, gradient-based update rules evaluated on adversarial examples minimize a robust form of the empirical risk function at a rate of $O(\ln(t)^2/t)$, where $t$ is the
number of iterations of the adversarial training process. 
This convergence rate mirrors the convergence of GD and SGD on the
standard empirical risk, {\it despite} the non-smoothness of the robust empirical risk function. 
We then use this analysis to quantify the
number of iterations required to obtain a given margin. 
We show that while GD may require exponentially many iterations to achieve large margin in non-adversarial training, adversarial training with (stochastic) gradient-based rules requires only polynomially many iterations to achieve large margin. 
We support our theoretical bounds with experimental results.


\subsection{Related Work}

Our results are most similar in spirit to \cite{ji2018risk}, which
uses techniques inspired by the Perceptron
\cite{novikoff1962convergence} to analyze the convergence of GD and
SGD on logistic regression. It derives a high probability $O(\ln(t)^2/t)$
convergence rate for SGD on separable data, as well as an $O(\ln(t)^2/t)$
convergence rate for GD in general. We adapt these techniques for
adversarial training. Our work also connects to work on ``implicit
bias'', which studies the parameter convergence of GD and SGD for
logistic regression on separable data \cite{gunasekar2018characterizing, gunasekar2018implicit, ji2018risk, nacson2019convergence, nacson2018stochastic, soudry2018implicit}. These works
show that the parameters generated by GD and SGD converge to the
parameters that correspond to the max-margin classifier at
polylogarithmic rates. This line of work, among other tools, employs techniques developed in the context of
 AdaBoost \cite{freund1997decision,
  mukherjee2013rate, telgarsky2012primal}. Our analysis is related in
particular to margin analyses of boosting \cite{rosset2004boosting, telgarsky2013margins}, which show that the path taken by boosting on
exponentially tailed losses approximates the max-margin classifier.

There is a large and active body of theoretical work on adversarial
robustness. While there are various hardness results in learning
robust models \cite{bubeck2018adversarial, gilmer2018adversarial, schmidt2018adversarially, tsipras2018robustness,
  tsipras2018there},
our analysis shows that such results may not apply to practical
settings. Our analysis uses a robust optimization lens previously
applied to machine learning in work such as \cite{sinha2017certifiable, caramanis201214, xu2009robustness}. While
\cite{xu2009robustness} shows that the max-margin classifier is the
solution to a robust empirical loss function, our work derives
explicit convergence rates for SGD and GD on such losses. Finally, we
note that adversarial training can be viewed as a data augmentation
technique. While the relation between margin and static data
augmentation was previously studied in \cite{rajput2019does}, our work
can be viewed as analyzing {\em adaptive} data augmentation methods.

\section{Overview}\label{sec:prelim}


		Let $\mX, \mY$, and $\mW$ denote the feature space, label space, and model space, respectively, and let $\ell: \mW \times \mX \times \mY \to \real_{\geq 0}$ be some loss function. Given a dataset $S \subseteq \mX \times \mY$, the empirical risk minimization objective is given by
		\begin{equation}\label{eq:emp_risk}
		\min_{w \in \mW} L(w) := \frac{1}{|S|}\sum_{(x,y) \in S}\ell(w,x,y).\end{equation}
		Unfortunately, generic empirical risk minimizers may not be robust to small input perturbations. To find models that are resistant to bounded input perturbations, we define the following robust loss functions
		\begin{equation}\label{eq:ell_robust}
		\ell_{\rob}(w,x,y) := \max_{\|\delta\| \leq \alpha} \ell(w,x+\delta,y),~~L_{\rob}(w) := \frac{1}{|S|}\sum_{(x,y)\in S}\ell_{\rob}(w,x,y).\end{equation}
		The function $L_{\rob}$ is a measure for the robustness of $w$ on $S$. While $\|\cdot\|$ can be any norm, here we focus on the  $\ell_2$ norm and let $\|\cdot\|$ denote it throughout our text. 
		Another important measure of robustness is {\it margin}. We focus on binary linear classification where $\mX = \real^d, \mY = \{\pm 1\}$, and $\mW = \real^d$. The  class predicted by $w$ on $x$ is given by $\sgn(w^Tx)$, and the margin of $w$ on $S$ is
		\begin{equation}\label{eq:margin}
		\margin_S(w) := \inf_{(x,y) \in S} \dfrac{y\langle w,x\rangle}{\|w\|}.\end{equation}		
		We say $w$ linearly separates $S$ if $\forall (x,y) \in S$, $y\langle w,x \rangle > 0$. Note $w$ linearly separates $S$ iff $\margin_S(w) > 0$. One can interpret margin as the size of the smallest perturbation needed to fool $w$ in to misclassifying an element of $S$. Thus, the most robust linear separator is the classifier with the largest margin, referred to as the {\it max-margin classifier}.

	\paragraph{Adversarial training.} 

		One popular class of defenses, referred to generally as {\it adversarial training} \cite{madry2017towards}, involve retraining a model on adversarially perturbed data points. The general heuristic follows two steps. At each iteration $t$ we construct adversarial examples for some subset of the training data. For each example $(x,y)$ in this set, an $\alpha$-bounded norm adversarial perturbation is constructed as follows:
		\begin{equation}\label{eq:adv_train}
		\delta^* = \argmax_{\|\delta\| \leq \alpha} \ell(w,x+\delta,y).\end{equation}
		We then update our model $w$ using an update rule $\mA$ that operates on the current model and ``adversarial examples'' of the form $(x+\delta^*,y)$. In the most general case, this update rule can also utilize true training data in $S$ and adversarial examples from prior iterations.

		More formally, let $w_0$ be our initial model. $S$ denotes our true training data, and $S'$ will denote all previously seen adversarial examples. We initialize $S' = \emptyset$. At each $t \geq 0$, we select some subset $S_t = \{(x_i^{(t)},y_i^{(t)})\}_{i=1}^m \subseteq S$. For $1 \leq i \leq m$, we let $\delta_i^{(t)}$ be the solution to \eqref{eq:adv_train} when $(x,y) = (x_i^{(t)},y_i^{(t)})$ and $w = w_t$. We then let
		$$S'_t = \{(x_i^{(t)}+\delta_i^{(t)},y_i^{(t)})\}_{i=1}^m$$
		$$S' = S' \cup S'_t.$$
		Thus, $S'_t$ is the set of adversarial examples computed at iteration $t$, while $S'$ contains all adversarial examples computed up to (and including) iteration $t$. Finally, we update our model $w_t$ via $w_{t+1} = \mA(w_t,S,S')$ for some update rule $\mA$. This generic notation will be useful to analyze a few different algorithms. A full description of adversarial training is given in Algorithm \ref{alg:adv_train}.

		\begin{algorithm}[htb]
			\caption{Adversarial training}
			{\small
			\begin{algorithmic}
				\STATE {\bfseries Input:} Training set $S$, perturbation size $\alpha$, update algorithm $\mA$, loss function $\ell(w,x,y)$.
				\STATE Initialize $w_0 \leftarrow 0$, $S' \leftarrow \emptyset$.
				\FOR{$t=0$ {\bfseries to} $T$}
					\STATE Select $S_t := \{(x_i^{(t)}, y_i^{(t)})\}_{i=1}^m \subseteq S$.
					\FOR{$i=1$ {\bfseries to} $m$}
						\STATE Set $\delta_i^{(t)} \leftarrow \argmax_{\|\delta\| \leq \alpha}\ell(w_{t},x_i^{(t)}+\delta_i^{(t)},y_{i}^{(t)})$.
					\ENDFOR
					\STATE Set $S'_t \leftarrow\{(x_i^{(t)}+\delta_{i}^{(t)}, y_{i}^{(t)} )\}_{i=1}^m$,~~$S' \leftarrow S' \cup S'_t$.
					\STATE Update $w_{t+1} \leftarrow \mA(w_{t},S,S')$.
				\ENDFOR
			\end{algorithmic}}
			\label{alg:adv_train}
		\end{algorithm}		

		Once $\alpha$ is fixed, there are two primary choices in selecting an adversarial training method: the subset $S_t$ used to find adversarial examples, and the update rule $\mA$. For example, one popular instance of adversarial training (discussed in detail in \cite{madry2017towards}) performs mini-batch SGD on the adversarial examples. Specifically, this corresponds to the setting where $S_t$ is randomly selected from $S$, and $\mA$ computes a mini-batch SGD update on $S'_t$ via
		\begin{equation}\label{eq:mini-batch-adv}
		w_{t+1} = \mA(w_t,S,S') = w_t - \dfrac{\eta_t}{|S'_t|} \sum_{(x+\delta,y) \in S'_t}\nabla \ell(w_t,x+\delta,y).\end{equation}

		In particular, this update does not utilize the full set $S'$ of all previously seen adversarial examples, but instead updates only using the set $S'_t$ of the most recently computed adversarial examples. It also does not use the true training samples $S$. However, other incarnations of adversarial training have used more of $S$ and $S'$ to enhance their accuracy and efficiency \cite{shafahi2018universal}.

	\paragraph{Main results.}		

		In the following, we analyze the performance of adversarial training for binary linear classification. In particular, we wish to understand how the choice of $\mA$, $\alpha$, and the number of iterations impact $\margin_S(w_t)$ and $L_{\rob}(w_t)$. We will make the following assumptions throughout:
		\begin{asm}\label{asm1} $\ell(w,x,y) = f(-y\langle w,x\rangle)$ where $f$ is nonnegative and monotonically increasing.\end{asm}
		\begin{asm}\label{asm2}$S$ is linearly separable with max-margin $\gamma \leq 1$.\end{asm}
		\begin{asm}\label{asm3}The parameter $\alpha$ satisfies $\alpha < \gamma$.\end{asm}
		\ref{asm1} guarantees that $\ell$ is a surrogate of the $0-1$ loss for linear classification, since $\ell(w,x,y)$ decreases as $y\langle w,x\rangle$ increases. \ref{asm2} allows us to compare the margin obtained by various methods to $\gamma$. We let $w^*$ denote the max-margin classifier. The assumption that $\gamma \leq 1$ is simply for convenience, as we can always rescale separable data to ensure this.

		Combined, \ref{asm2} and \ref{asm3} guarantee that at every iteration, $S \cup S'$ is linearly separable by $w^*$ with margin at least $\gamma-\alpha$, as we show in the following lemma. 
		\begin{lemma}\label{lem:lin_sep}
		Suppose \ref{asm2} and \ref{asm3} hold, and let $w^*$ be the max-margin classifier of $S$. Then at each iteration of Algorithm \ref{alg:adv_train}, $w^*$ linearly separates $S\cup S'$ with margin at least $\gamma-\alpha$.\end{lemma}

		\begin{proof}By construction, any element in $S \cup S'$ is of the form $(x+\delta,y)$ where $(x,y) \in S$ and $\|\delta\|\leq \alpha$. By assumption on $w^*$ and the Cauchy-Schwarz inequality,
		$$y\langle w^*, x+\delta\rangle = y\langle w^*,x\rangle + y\langle w^*,\delta\rangle \geq \gamma-\alpha.$$\end{proof}

			We can now state the main theorems of our work. We first show that adversarial training may take a long time to converge to models with large margin, even when $\mA$ finds an empirical risk minimizer (ERM) of the $0-1$ loss on $S \cup S'$. Note that by Lemma \ref{lem:lin_sep}, this is equivalent to finding a linear separator of $S \cup S'$. That is, even if $\mA$ finds a model that perfectly fits the training data {\it and} all adversarial examples at each step, this is insufficient for fast convergence to good margin.
			\begin{theorem}[Informal]Suppose $\mA(w_t,S,S')$ outputs a linear separator of $S \cup S'$. In the worst case, Algorithm \ref{alg:adv_train} requires $\Omega(\exp(d\epsilon^2/\alpha^2))$ iterations to achieve margin $\epsilon$.\end{theorem}

			We then show that for logistic regression, if $\mA$ performs a full-batch gradient descent update on the adversarial examples, then adversarial training quickly finds a model with large margin. This corresponds to the setting where $\mA$ is given as in \eqref{eq:mini-batch-adv} with $S_t = S$. We refer to this as GD with adversarial training. 
			\begin{theorem}[Informal]Let $\{w_t\}_{t \geq 1}$ be the iterates of GD with adversarial training. Then $L_{\rob}(w_t) \leq \tmO(1/t)$, and for $t = \Omega(\poly((\gamma-\alpha)^{-1}))$, $\margin_S(w_t) \geq \alpha$.\end{theorem}
			
			The $\tmO$ notation hides polylogarithmic factors. By contrast, one can easily adapt lower bounds in \cite{gunasekar2018implicit} on the convergence of gradient descent to the max-margin classifier to show that standard gradient descent requires $\Omega(\exp((\gamma-\alpha)^{-1}))$ iterations to guarantee margin $\alpha$.
			
			Since the inner maximization in Algorithm \ref{alg:adv_train} is often expensive, we may want $S_t$ to be small. When $|S_t| = 1$ and $\mA$ performs the gradient update in \eqref{eq:mini-batch-adv}, Algorithm \ref{alg:adv_train} becomes SGD with adversarial training, in which case we have the following.
			\begin{theorem}[Informal]Let $\{w_t\}_{t\geq 1}$ be the iterates of SGD with adversarial training, and let $\hat{w}_t = (\sum_{j < t} w_j)/t$. With probability at least $1-\delta$, $L_{\rob}(\hat{w}_t) \leq \tmO(\ln(1/\delta)/t)$ and if $t \geq \Omega(\poly((\gamma-\alpha)^{-1},\ln(1/\delta)))$, then $\margin_S(\hat{w}_t) \geq \alpha$.\end{theorem}

\section{Fundamental Limits of Adversarial Training for Linear Classifiers}\label{sec:fund_limits}

	We will now show that even if the subroutine $\mA$ in Algorithm \ref{alg:adv_train} outputs an arbitrary empirical risk minimizer (ERM) of the $0-1$ loss on $S \cup S'$, then in the worst-case $\Omega(e^d)$ iterations are required to obtain margin $\epsilon$. 
	
	Suppose that $\mA$ in Algorithm \ref{alg:adv_train} is defined by
	$$\mA(w_t,S,S') \in \argmin_{w \in \real^d} \sum_{(x,y) \in S \cup S'} \ell_{0-1}(\sgn(w^Tx), y).$$
	By Lemma \ref{lem:lin_sep}, $S \cup S'$ is linearly separable. Thus, the update $\mA(w_t,S,S')$  is equivalent to finding some linear separator of $S\cup S'$. When $\mA$ is an arbitrary ERM solver, we can analyze the worst case convergence of adversarial training by viewing it as a game played between two players. At each iteration, Player 1 augments the current data with adversarial examples computed for the current model. Player 2 then tries to find a linear separator of all previously seen points with small margin. This specialization of Algorithm \ref{alg:adv_train} is given in Algorithm \ref{alg:adv_train_2}.

	\begin{algorithm}[ht]
		\caption{Adversarial training with an arbitrary ERM}
		{\small
		\begin{algorithmic}
			\STATE {\bfseries Input:} Training set $S$, perturbation size $\alpha$, loss function $\ell(w,x,y)$.
			\STATE Initialize $w_0 \leftarrow 0, S' \leftarrow \emptyset$.
			\FOR{$t=0$ {\bfseries to} $T$}
				\STATE Select $S_t := \{(x_i^{(t)}, y_i^{(t)})\}_{i=1}^m \subseteq S$.
				\FOR{$i=1$ {\bfseries to} $n$}
					\STATE Set $\delta_i^{(t)} \leftarrow \argmax_{\|\delta\| \leq \alpha}\ell(w_{t},x_i^{(t)}+\delta_i^{(t)},y_{i}^{(t)})$. \hfill\%Player 1's move
				\ENDFOR
				\STATE Set $S'_t \leftarrow\{(x_i^{(t)}+\delta_{i}^{(t)}, y_{i}^{(t)} )\}_{i=1}^m$,~~$S' \leftarrow S' \cup S'_t$.
				\STATE $w_{t+1}$ is set to be some linear separator of $S \cup S'$. \hfill\%Player 2's move
			\ENDFOR
		\end{algorithmic}}
		\label{alg:adv_train_2}
	\end{algorithm}	

	In the following, we assume $S_t = S$ for all $t$. This only reduces reduces the ability of the worst-case ERM solver to output some model with small margin.
	We say a sequence $\{w_t\}_{t=1}^T$ is {\it admissible} if is generated according to $T$ iterations of Algorithm \ref{alg:adv_train_2}. Intuitively, the larger $T$ is (\ie, the more this game is played), the more restricted the set of linear separators of $S\cup S'$ becomes. We might hope that after a moderate number of rounds, the only feasible separators left have high margin with respect to the original training set $S$.

	We show that this is not the case. Specifically, an ERM may still be able to output a linear separator with margin at most $\epsilon$, even after exponentially many iterations of adversarial training.

	\begin{theorem}\label{thm:adv_main}Let $S = \{(\gamma v,1), (-\gamma v,-1)\}$, where $v$ is a unit vector in $\real^d$. Then, there is some constant $c$ such that for any $\epsilon \leq \alpha$, there is an admissible sequence $\{w_t\}_{t \geq 0}$ such that $\margin_S(w_t) \leq \epsilon$ for all $t$ satisfying
	$$t \leq \dfrac{1}{2}\exp\left(\dfrac{c(d-1)\epsilon^2}{(\gamma+\epsilon)^2}\right).$$
	\end{theorem}

	The proof proceeds by relating the number of times an ERM can obtain margin $\epsilon$ to the size of {\it spherical codes}. These are arrangements of points on the sphere with some minimum angle constraint \cite{delsarte1991spherical, kabatyanskii1974bounds, delsarte1972bounds, sloane1981tables} and have strong connections to sphere packings and lattice density problems \cite{conway2013sphere}. We show how an arbitrary ERM can use a spherical code of size $m$ to generate an admissible sequence with small margin for the first $m$ iterations. While computing spherical codes of maximal size is a notoriously difficult task \cite{cohn2014sphere}, spherical codes with $\Omega(\exp(d))$ points can be constructed with high probability by taking spherically symmetric points on the sphere at random. A full proof can be found in Appendix \ref{sec:fund_proof}.

	This implies that even for relatively small $\epsilon$, the number of times an ERM can achieve margin $\epsilon \ll$ is $\Omega(\exp(d))$ in the worst-case. As we will show in the proceeding sections, this worst-case scenario is overcome when we combine adversarial training with gradient dynamics.

\section{Adversarial Training with Gradient-based Updates}\label{sec:grad_adv_train}

	We will now discuss gradient-based versions of adversarial training, in which we use gradients evaluated with respect to adversarially perturbed training points to update our model.
	Suppose that $S = \{(x_i,y_i)\}_{i=1}^n$ has associated empirical risk function $L$ as in \eqref{eq:emp_risk}. Let $w_0$ be some initial model.
	In adversarial training with gradient methods, at each $t \geq 0$, we select $S_t = \{(x_i^{(t)},y_i^{(t)})\}_{i=1}^m \subseteq S$ and update via
	\begin{equation}\label{gd_adv_train_1}
	\delta^{(t)}_i = \argmax_{\|\delta\| \leq \alpha}~\ell(w_t,x_i^{(t)}+\delta,y_i^{(t)}),\;\;\forall i\in [m]\end{equation}
	\begin{equation}\label{gd_adv_train_2}
	w_{t+1} = w_t - \frac{\eta_t}{|S_t|} \sum_{i=1}^m \nabla \ell(w_t,x_i^{(t)}+\delta_i^{(t)},y_i^{(t)})\end{equation}
	where $\eta_t$ is the step size and $\delta_i^{(t)}$ is treated as constant with respect to $w_t$ when computing the gradient $\nabla \ell(w_t,x_i^{(t)}+\delta_i^{(t)},y_i^{(t)})$. When $S_t = S$, we refer to this procedure as $\alpha$-GD. When $S_t$ is a single sample selected uniformly at random, we refer to this procedure as $\alpha$-SGD. Note that when $\alpha = 0$, this becomes standard GD and SGD on $L$.

	Note that both these methods are special cases of Algorithm \ref{alg:adv_train}, where the update $w_{t+1} = \mA(w_t,S,S')$ is given by \eqref{gd_adv_train_2}. Before we proceed, we present an alternate view of this method. Recall the functions $\ell_{\rob}$ and $L_{\rob}$ defined in \eqref{eq:ell_robust}. To understand $\alpha$-GD, we will use Danskin's theorem \cite{danskin2012theory}. We note that this was previously used in \cite{madry2017towards} to justify adversarial training with gradient updates. The version we cite was shown by Bertsekas \cite{bertsekas1971control}. A more modern proof can be found in \cite{bertsekas1997nonlinear}.

	\begin{prop}[Danskin]\label{prop:danskin}Suppose $X$ is a non-empty compact topological space and $g: \real^d \times X \to \real$ is a continuous function such that $g(\cdot, \delta)$ is differentiable for every $\delta \in X$. Define
	$$\delta^*(w) = \{\delta \in \argmax_{\delta \in X} g(w,\delta)\},~~\psi(w) = \max_{\delta \in X}g(w,\delta).$$
	Then $\psi$ is subdifferentiable with subdifferential given by $\partial \psi(w) = \conv(\left\{ \nabla_w g(w,\delta)~\middle| \delta \in \delta_w^* \right\})$.
	\end{prop}

	Thus, we can compute subgradients of $L_{\rob}$ by solving the inner maximization problem \eqref{gd_adv_train_1} for each $i \in [n]$, and then taking a gradient. In other words, for a given $w$, let $\delta_i$ be a solution to \eqref{gd_adv_train_1}. Then $\nabla_w\ell(w,x_i+\delta_i,y_i) \in \partial \ell_{\rob}(w,x,y)$.
	Therefore, $\alpha$-GD is a subgradient descent method for $L_{\rob}$, while $\alpha$-SGD is a stochastic subgradient method. Furthermore, if the solution to \eqref{gd_adv_train_1} is unique then Danskin's theorem implies that $\alpha$-GD actually computes a gradient descent step, while $\alpha$-SGD computes a stochastic gradient step. Indeed, the above proposition also motivated \cite{madry2017towards} and \cite{shaham2018understanding} to use a projected gradient inner step to compute adversarial examples and approximate adversarial training with SGD.
	
	For linear classification, we can derive stronger structural connections between $\ell$ and $\ell_{\rob}$.
	\begin{lemma}\label{lem:linear_robust_loss}
		Suppose $\ell(w,x,y) = f(-y\langle w, x\rangle)$ for $f$ monotonically increasing and differentiable. Then, the following properties hold:
		\begin{enumerate}[{(1)}]
			\item For all $w$, $\ell_{\rob}(w,x,y)$ satisfies $\ell_{\rob}(w,x,y) = f(-y\langle w, x\rangle +\alpha \|w\|)$.
			\item For all $w$, $\ell_{\rob}(w,x,y)$ is subdifferentiable with $f'(-y\langle w,x \rangle + \alpha\|w\|)(-yx + \overline{w}) \in \partial \ell_{\rob}(w,x,y)$, where $\overline{w} = w/\|w\|$, if $w \neq 0$ and $0$ otherwise. 
			\item If $f$ is strictly increasing, then $\ell_{\rob}(w,x,y)$ is differentiable at all $w \neq 0$.
			\item If $f$ is $M$-Lipschitz, $\beta$-smooth, and strictly increasing, then $\ell_{\rob}(w,x,y)$ is twice differentiable at $w\neq 0$, in which case $\nabla^2 \ell_{\rob}(w,x,y) \preceq \beta' I$, where $\beta' = \alpha M/\|w\| + \beta(\|x\|+\alpha)^2$.
			\item If $f$ is convex, then $\ell_{\rob}(w,x,y)$ is convex.
		\end{enumerate}
	\end{lemma}

	A full proof is given in Appendix \ref{sec:main_lem_proof}. Thus, if $f$ is convex, then $L_{\rob}(w)$ is convex and $\alpha$-GD and $\alpha$-SGD perform (stochastic) subgradient descent on a convex, non-smooth function. Unfortunately, even if $L(w)$ is smooth, $L_{\rob}(w)$ is typically non-smooth. Standard results for convex, non-smooth optimization then suggest that $\alpha$-GD and $\alpha$-SGD obtain a convergence rate of $\mO(1/\sqrt{t})$ on $L_{\rob}(w)$. However, this is a pessimistic convergence rate for subgradient methods on non-smooth convex functions. By Lemma \ref{lem:linear_robust_loss}, $L_{\rob}$ inherits many nice geometric properties from $L$. There is therefore ample reason to believe the pessimistic $\mO(1/\sqrt{t})$ convergence rate is not tight. As we show in the following, $\alpha$-GD and $\alpha$-SGD actually minimize $L_{\rob}$ at a much faster $O(\log^2(t)/t)$ rate. 

	In the next section, we analyze the convergence of $\alpha$-GD and $\alpha$-SGD, measured in terms of $L_{\rob}(w_t)$, as well as $\margin_S(w_t)$, for logistic regression. We adapt the classical analysis of the Perceptron algorithm from \cite{novikoff1962convergence} to show that a given margin is obtained. To motivate this, we first analyze an adversarial training version of the Perceptron.

	\subsection{Adversarial Training with the Perceptron}

		Let $f(u) = \max\{0,u\}$. Then $\ell(w,x,y) = f(-y\langle w,x\rangle) = \max\{0,-y\langle w,x \rangle\}$. For notational convenience, suppose that for all $(x,y) \in S$, $\|x\| \leq 1$. Let $w_0 = 0$. Applying SGD with step-size $\eta = 1$, we get updates of the form $w_{t+1} = w_t + g_t$ where $g_t = y_tx_t$ if $y_t\langle w_t,x_t\rangle \leq 0$ and $0$ otherwise. This is essentially the  Perceptron algorithm, in which case \cite{novikoff1962convergence} implies the following.

		\begin{lemma}This procedure stops after at most $(1/\gamma)^2$ non-zero updates, at which point $w_t$ linearly separates $S$.\end{lemma}

		Suppose we instead perform $\alpha$-SGD with step-size $\eta = 1$ and $w_0 = 0$. Given $w_t$, let $\overline{w}_t = w_t/\|w_t\|$ if $w_t \neq 0$ and $0$ otherwise. Lemma \ref{lem:linear_robust_loss} implies that $\alpha$-SGD does the following: Sample $i_t \sim [n]$ uniformly at random, then update via
		$$w_{t+1} = w_t + \begin{cases}
		y_{i_t}x_{i_t} - \alpha\overline{w}_t &\mbox{, }~y_{i_t}\langle w_t,x_{i_t}\rangle - \alpha\|w_t\| \leq 0\\
		0&\mbox{, otherwise.}
		\end{cases}$$

		Due to its resemblance to the Perceptron, we refer to this update as the $\alpha$-Perceptron. We then get an analogous result on the number of iterations required to find classifiers with a given margin.

		\begin{lemma}
		The $\alpha$-Perceptron stops after at most $\left(\frac{1+\alpha}{\gamma-\alpha}\right)^2$ non-zero updates, after which point $w_t$ has margin at least $\alpha$.
		\end{lemma}

		\begin{proof}
			Assume the update at $t$ is non-zero, so $y_{i_t} \langle w_t, x_{i_t} \rangle - \alpha\|w_t\| \leq 0$. Let $w^*$ be a unit vector that achieves margin $\gamma$. Then,
			\begin{align*}
				\langle w_{t+1}-w_t,w^*\rangle &= \langle y_{i_t}x_{i_t} -\alpha\overline{w}_t, w^*\rangle = \langle y_{i_t}x_{i_t}, w^*\rangle - \alpha \langle\overline{w}_t, w^*\rangle \geq \gamma - \alpha.
			\end{align*}
			Therefore, after $T$ iterations, $\langle w_T, w^*\rangle \geq T(\gamma-\alpha)$. Next, we upper bound $\|w_{t+1}\|$ via:
			\begin{align*}
				\|w_{t+1}\|^2 = \|w_t\|^2 + 2 (y_{i_t} \langle w_t, x_{i_t}\rangle - \alpha \|w_t\|) + \|y_{i_t}x_{i_t} - \alpha \overline{w}_t\|^2 \leq \|w_t\|^2 + (1+\alpha)^2.
			\end{align*}
			The last step follows from the fact that we update iff $y_{i_t} \langle w_t, x_{i_t}\rangle - \alpha \|w_t\| \leq 0$. Recursively, we find that $\|w_{T}\|^2 \leq T(1+\alpha)^2$, so $\|w_{T}\| \leq \sqrt{T}(R+\alpha)$.
			Combining the above,
			\[
				1 \geq \dfrac{\langle w_T, w^*\rangle}{\|w_T\|\|w^*\|} \geq \dfrac{\sqrt{T}(\gamma-\alpha)}{1+\alpha}\\
				 \implies T \leq \left(\frac{1+\alpha}{\gamma-\alpha}\right)^2.
			\]
			The update at $(x,y)$ is non-zero iff $w_t$ has margin $\leq \alpha$ at $(x,y)$, so once $\alpha$-Perceptron stops updating, $\margin_S(w_T) \geq \alpha$.
		\end{proof}	
	
	While simple, this result hints at an underlying, more general phenomenon for linearly separable datasets: The convergence of gradient-based adversarial training to a robust risk minimizer often mirrors the convergence of conventional gradient methods to an empirical risk minimizer. We demonstrate this principle formally in the following section for logistic regression.

\section{Adversarial Training for Logistic Regression}\label{sec:adv_train_lr}

	We will now analyze the convergence and margin of $\alpha$-GD and $\alpha$-SGD for logistic regression. In logistic regression, $\ell(w,x,y) = f(-y\langle w, x\rangle)$ where $f(u) = \ln(1+\exp(u))$. Note that $f$ is convex, $1$-Lipschitz, and $1$-smooth, and bounded below by 0. For notational simplicity, suppose that $S = \{(x_i,y_i)\}_{i=1}^n$ with $\|x_i\| \leq 1$ for all $i$. Thus, the max-margin $\gamma$ of $S$ satisfies $\gamma \leq 1$.

\subsection{Convergence and Margin of $\alpha$-GD}\label{sec:gd_log_reg}

	Let $\{w_t\}_{t \geq 0}$ be the iterates of $\alpha$-GD with step-sizes $\{\eta_t\}_{t \geq 0}$. We will suppose that $w_0 = 0$, and $\eta_0 = 1$. These assumptions are not necessary, but simplify the statement and proofs of the following results. Full proofs of all results in this section can be found in Appendix \ref{sec:gd_proof}.

	To analyze the convergence of $\alpha$-GD on $L_{\rob}$, we will use the fact that by Lemma \ref{lem:linear_robust_loss}, while $L_{\rob}$ is not smooth, it is $\beta$-smooth away from $0$. We then use a Perceptron-style argument inspired by \cite{ji2018risk} to show that after a few iterations, the model $w_t$ produced by $\alpha$-GD has norm bounded below by some positive constant. We can then apply standard convergence techniques for gradient descent on $\beta$-smooth functions to derive the following.

	\begin{theorem}\label{thm:gd_log_reg1}
		Suppose $w_0 = 0, \eta_0 = 1$, and $\forall t \geq 1$, $\eta_t \leq \left(\frac{2\alpha}{(\gamma-\alpha)} + (1+\alpha)^2\right)^{-1}$.
		Then $\forall t \geq 2$,
		$$L_{\rob}(w_t) \leq \frac{1}{t} + \bigg(\sum_{j=1}^{t-1}\eta_j\bigg)^{-1}\left(\frac{1}{4} + \frac{\ln(t)^2}{(\gamma-\alpha)^2}\right).$$
	\end{theorem}

	We can use the above results to show that after a polynomial number of iterations, we obtain a model with margin $\alpha$. To do so, we first require a straightforward lemma relating $L_{\rob}$ to margin.

	\begin{lemma}\label{lem:log_reg_margin}
	If $L_{\rob}(w) \leq \frac{\ln(2)}{n}$ then $\margin_S(w) \geq \alpha$.\end{lemma}

	We then get the following.

	\begin{cor}\label{thm:gd_log_reg2}
	Suppose that for $t \geq 1$, $\eta_t = \eta \leq 1$ and $\eta_t \leq \left(\frac{2\alpha}{(\gamma-\alpha)} + (1+\alpha)^2\right)^{-1}$. For all $q > 1$, there is a constant $C_q$ such that $\margin_S(w_t) \geq \alpha$ for all $t$ satisfying
	\begin{equation}\label{eq:gd_log_reg_margin}
	t \geq \max\left\{C_q,\left(\dfrac{n}{\eta(\gamma-\alpha)^2\ln(2)}\right)^{q}\right\}.\end{equation}
	\end{cor}

	Ignoring all other terms, this implies that for all $q > 1$, $t = O((\gamma-\alpha)^{-2q})$ iterations of $\alpha$-GD sufficient to obtain margin $\alpha$. The constant $C_q$ is how large $T$ must be so that for all $t \geq T$, $\ln(t)/t < t^{-1/q}$. As such, the constant $C_q$ tends to $\infty$ as $q$ tends to $1$.

	On the other hand, one can show that standard gradient descent may require exponentially many iterations to reach margin $\alpha$, even though it eventually converges to the max-margin classifier. This follows immediately from a direct adaptation of lower bounds from \cite{gunasekar2018implicit}.

	\begin{theorem}\label{thm:exp_gd}
	Let $x = (1,0), y = 1$. Let $(w_t)_{t \geq 1}$ be the iterates of GD with constant step-size $\eta = 1$ initialized at $w_0 = (0,c)$ for $c > 0$. For all $t < \exp(c/(1-\alpha))$, $\margin_{(x,y)}(w_t) < \alpha$.
	\end{theorem}

	One can show that as $\eta$ decreases, this convergence rate only decreases. Thus, the exponentially slow convergence in margin is not an artifact of the choice of step-size, but rather an intrinsic  property of gradient descent on logistic regression.



\subsection{Convergence and Margin of $\alpha$-SGD}\label{sec:sgd_log_reg}

	Recall that at each iteration $t$, $\alpha$-SGD selects $i_t \sim [n]$ uniformly at random and updates via $w_{t+1} = w_t-\eta_t\nabla \ell_{\rob}(w_t,x_{i_t},y_{i_t})$. We would like to derive similar results to those for $\alpha$-GD above. While we could simply try to derive the same results by taking expectations over the iterates of $\alpha$-SGD, this ignores relatively recent work that has instead derived high-probability convergence results for SGD \cite{ji2018risk, rakhlin2011making}. In particular, \cite{ji2018risk} uses a martingale Bernstein bound from \cite{beygelzimer2011contextual} to derive a high probability $O(\ln(t)^2/t)$ convergence rate for SGD on separable data. While the analysis cannot be used directly, we use the structural connections between $\ell$ and $\ell_{\rob}$ in Lemma \ref{lem:linear_robust_loss} to adapt the techniques therein. We derive the following:

	\begin{theorem}\label{thm:log_reg1}
	Let $\{w_t\}_{t \geq 0}$ be the iterates of $\alpha$-SGD with constant step size $\eta \leq \min\{1,2(1+\alpha)^{-2}\}$ and $w_0 = 0$. For any $t \geq 1$, with probability at least $1-\delta$, $\hat{w}_t := \frac{1}{t} \sum_{j < t} w_j$ satisfies
	$$L_{\rob}(\hat{w}_t) \leq \dfrac{1}{\eta t}\left(\dfrac{4\ln(t)}{\gamma-\alpha}+6\right)\left(\dfrac{8\ln(t)}{(\gamma-\alpha)^2} + \frac{8}{\gamma-\alpha} +4\ln(1/\delta)\right).$$
	\end{theorem}

	A similar (but slightly more complicated) result can be shown when $w_0 \neq 0$, which we have omitted for the sake of exposition. Using Lemma \ref{lem:log_reg_margin}, we can now show that after $t \geq \poly(n,\eta^{-1},(\gamma-\alpha)^{-1},\ln(1/\delta))$ iterations, with high probability, $\hat{w}_t$ will have margin at least $\alpha$. 

	\begin{cor}\label{thm:log_reg2}
	Let $\{w_t\}_{t\geq 0}$ be the iterates of $\alpha$-SGD with constant step size $\eta \leq \min\{1,2(1+\alpha)^{-2}\}$ and $w_0 = 0$. For all $q > 1$, there is a constant $C_q$
	$$t \geq \max\left\{C_q,\left[ \dfrac{cn}{\eta}\left(\dfrac{1}{(\gamma-\alpha)^3} + \dfrac{\ln(1/\delta)}{\gamma-\alpha}\right)\right]^{q}\right\}$$
	then with probability at least $1-\delta$, $\margin_S(\hat{w}_t) \geq \alpha$. Here, $c$ is some universal constant.
	\end{cor}

	Ignoring all other factors, this implies that for any $q > 1$, with high probability $O((\gamma-\alpha)^{-2q})$ iterations of $\alpha$-SGD are sufficient to obtain margin $\alpha$. As with $\alpha$-GD, the constant $C_q$ is how large $T$ must be so that for all $t \geq T$, $\ln(t)/t \leq t^{-1/q}$. Proofs of the above results can be found in Appendix \ref{sec:sgd_proof}.


\section{Experiments}\label{sec:experiments}
    To corroborate our theory, we evaluate $\alpha$-GD and $\alpha$-SGD on logistic regression with linearly separable data. As in our theory, we train linear classifiers $w$ whose prediction on $x$ is $\hat{y} = \sgn(w^Tx)$. We compare $\alpha$-GD and $\alpha$-SGD for various values of $\alpha$. Note that when $\alpha = 0$, $\alpha$-GD and $\alpha$-SGD are identical to the standard GD and SGD training algorithms, which we use as benchmarks.

    \paragraph{Evaluation metrics.} We evaluate these methods in the three ways. First, we compute the training loss $L(w_t)$ in \eqref{eq:emp_risk}. Second, we compute the margin $\margin_S(w_t)$ in \eqref{eq:margin}. To aid clarity, we plot the {\it truncated margin}, $\margin_S^+(w_t) := \max\{0, \margin_S(w_t)\}$. Third, we plot the robust training loss $L_{\rob}(w_t)$ in \eqref{eq:ell_robust}. This is governed by $\alpha$. For convenience, we refer to this as the $\alpha$-robust loss and denote it by $L_{\alpha}(w_t)$. To compare $\alpha$-SGD for different values of $\alpha$, we plot $L_{\alpha}(w_t)$ for $\alpha$-SGD. In particular, standard GD and SGD correspond to $\alpha = 0$, in which case we plot $L_{0}(w_t) = L(w_t)$.

    \paragraph{Setup and implementation.} All experiments were implemented in PyTorch. We vary $\alpha$ over $\{0,0.25,0.5,0.75\}$. When $\alpha = 0$, we get standard GD and SGD. In all experiments, we use a constant step-size $\eta$ that is tuned for each $\alpha$. The tuning was done by varying $\eta$ over $\{0.1/2^k | 0 \leq k < 10\}$, evaluating the average value of $L_{\alpha}(w_t)$ after $500$ iterations, and selecting the step-size with the smallest loss. For $\alpha$-SGD, we did the same, but for $L_{\rob}(w_t)$ averaged over 5 trials. When plotting the above evaluation metrics for $\alpha$-SGD, we ran multiple trials (where the number varied depending on the dataset) and plotted the average, as well as error bars corresponding to the standard deviation.

    \paragraph{Synthetic data.} We draw $x \in \real^2$ uniformly at random from circles of radius 1 centered at $(2,0)$ and $(-2,0)$. These correspond to $+1$ and $-1$ labeled points, respectively. We draw $50$ points from each circle, and also add the points $(e_1,1)$ and $(-e_1,-1)$, where $e_1 = [1,0]^T$. This guarantees that the max-margin is $\gamma =1$. We initialize at $w_0 = [0,1]^T$. While we observe similar behavior for any reasonable initialization, this intialization is used to compare how the methods ``correct'' bad models. For $\alpha$-SGD, we computed the average and standard deviation of the evaluation metrics above over 5 trials.

    \paragraph{Real data.} We use the Iris Dataset \cite{Dua:2019}, which contains data for 3 classes, Iris-setosa, Iris-versicolor, and Iris-virginica. Iris-setosa is linearly separable from Iris-virginica with max-margin $\gamma \approx 1.22$. We initialize $w_0$ with entries drawn from $\mathcal{N}(0,1)$. We found that our results were not especially sensitive to the initialization scheme. While different initializations result in minor changes to the plots below, the effects were consistently uniform across different $\alpha$. For $\alpha$-SGD, we computed the average and standard deviation of the evaluation metrics above over 9 trials. Note that we increased the number here due to the increased variance of single-sample SGD on this dataset over the synthetic dataset above.

    \begin{figure}[H]
        \centering
        \begin{subfigure}[b]{0.32\textwidth}
            \includegraphics[width=\linewidth]{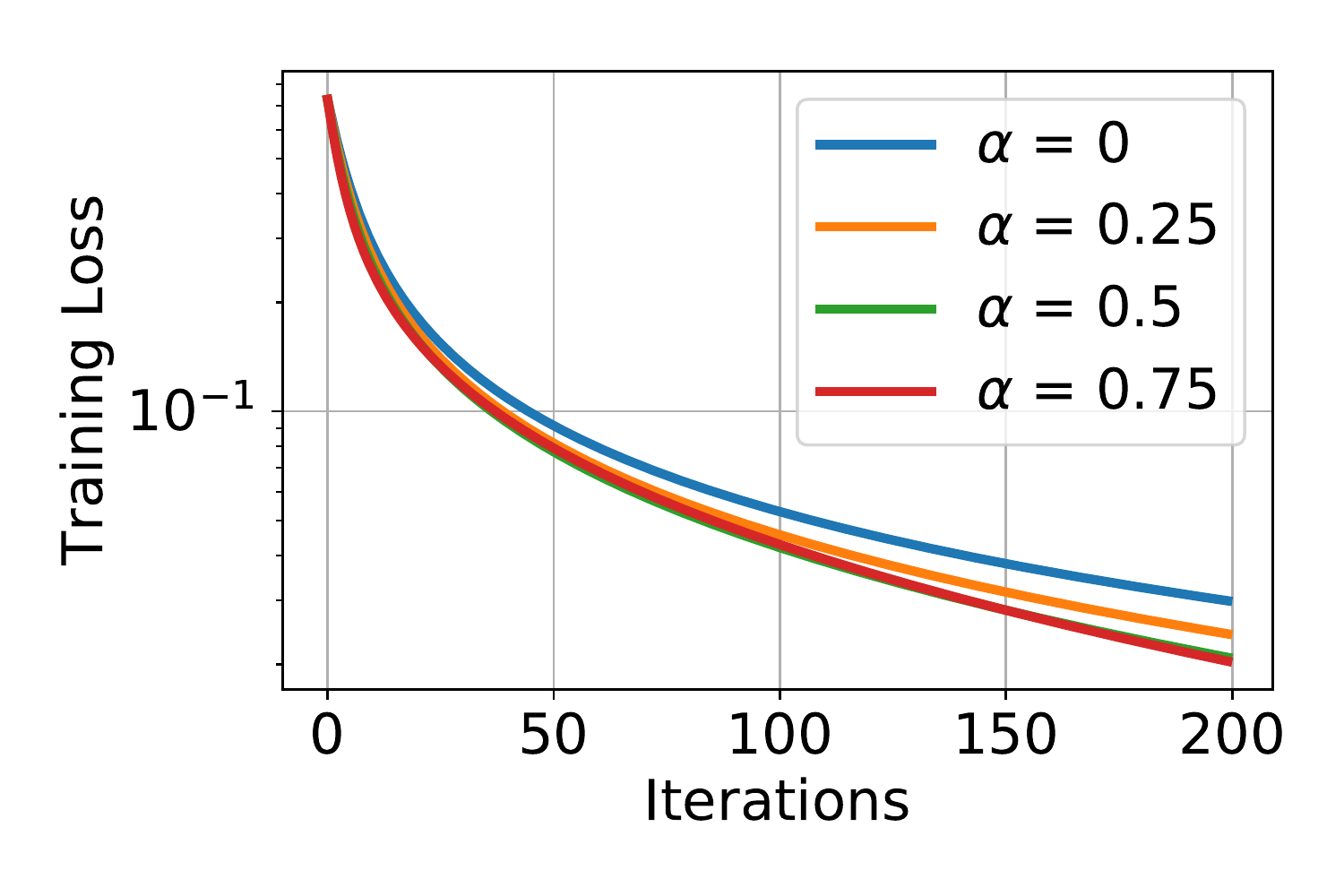}
            \caption{Training Loss}
        \end{subfigure}
        ~ 
        \begin{subfigure}[b]{0.32\textwidth}
            \includegraphics[width=\linewidth]{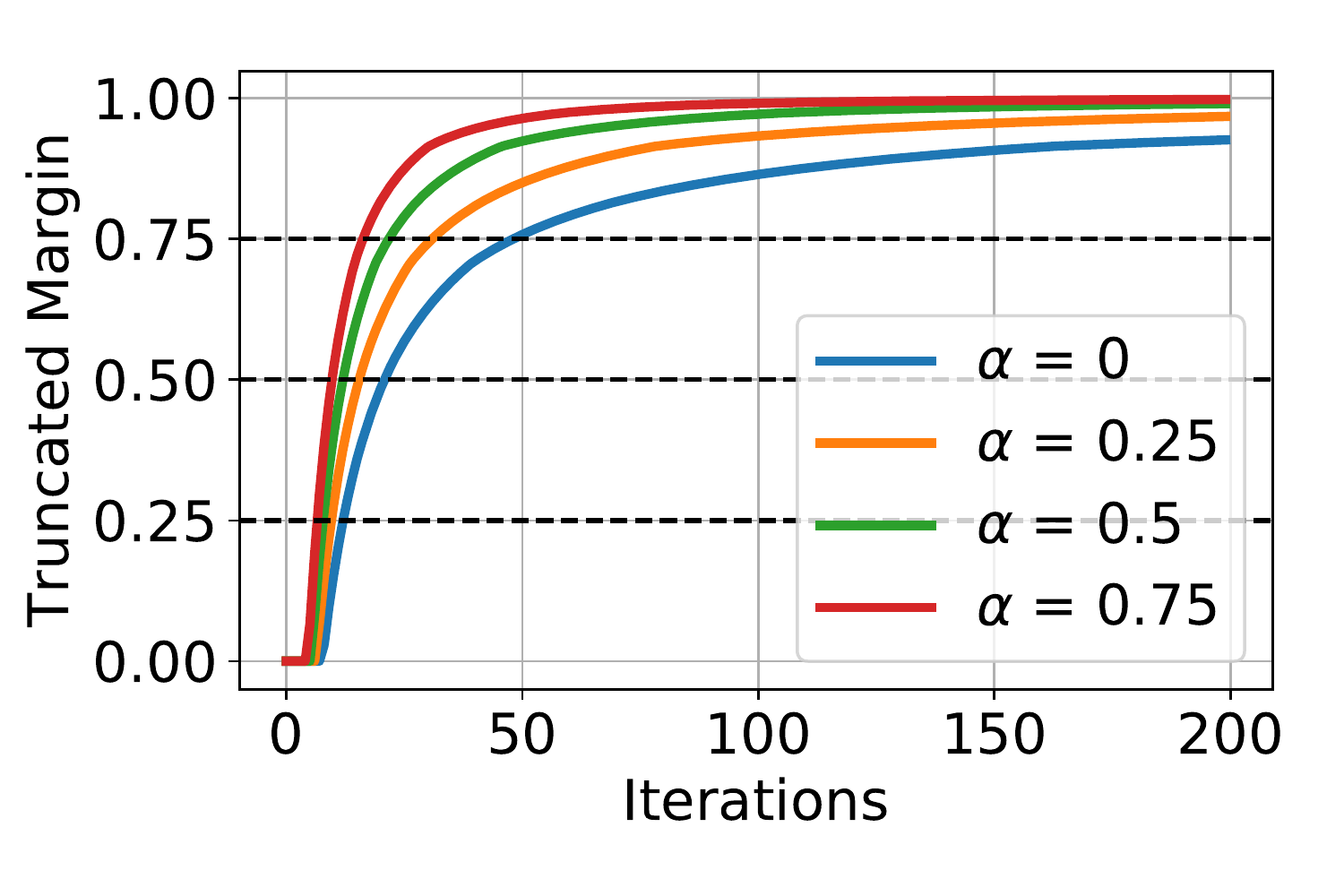}
            \caption{Truncated Margin}
        \end{subfigure}
        ~
        \begin{subfigure}[b]{0.32\textwidth}
            \includegraphics[width=\linewidth]{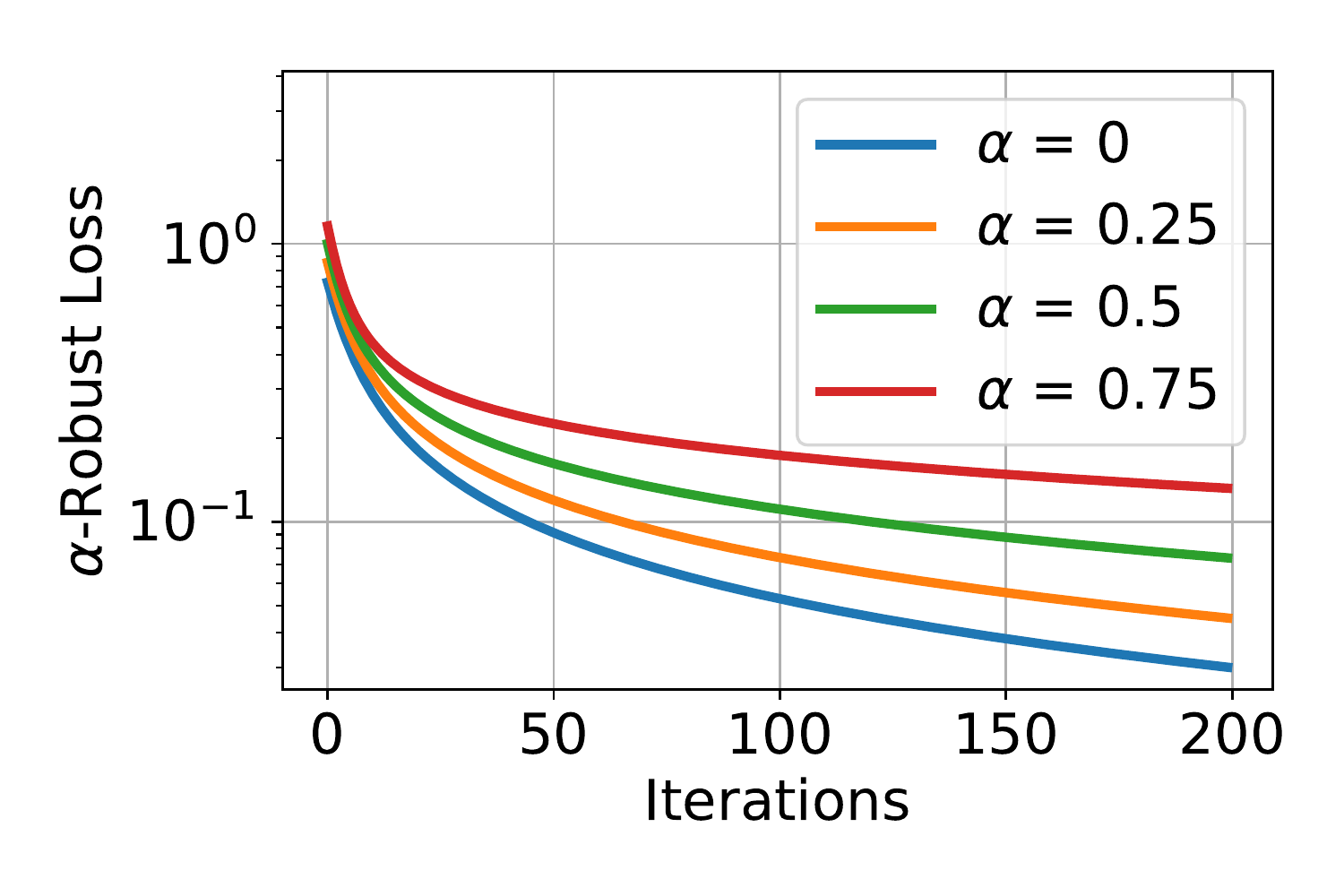}
            \caption{$\alpha$-Robust Loss}
        \end{subfigure}
        \caption{Results for $\alpha$-GD on the synthetic dataset.}
        \label{fig:synth_gd}
    \end{figure}    

    \begin{figure}[H]
        \centering
        \begin{subfigure}[b]{0.32\textwidth}
            \includegraphics[width=\linewidth]{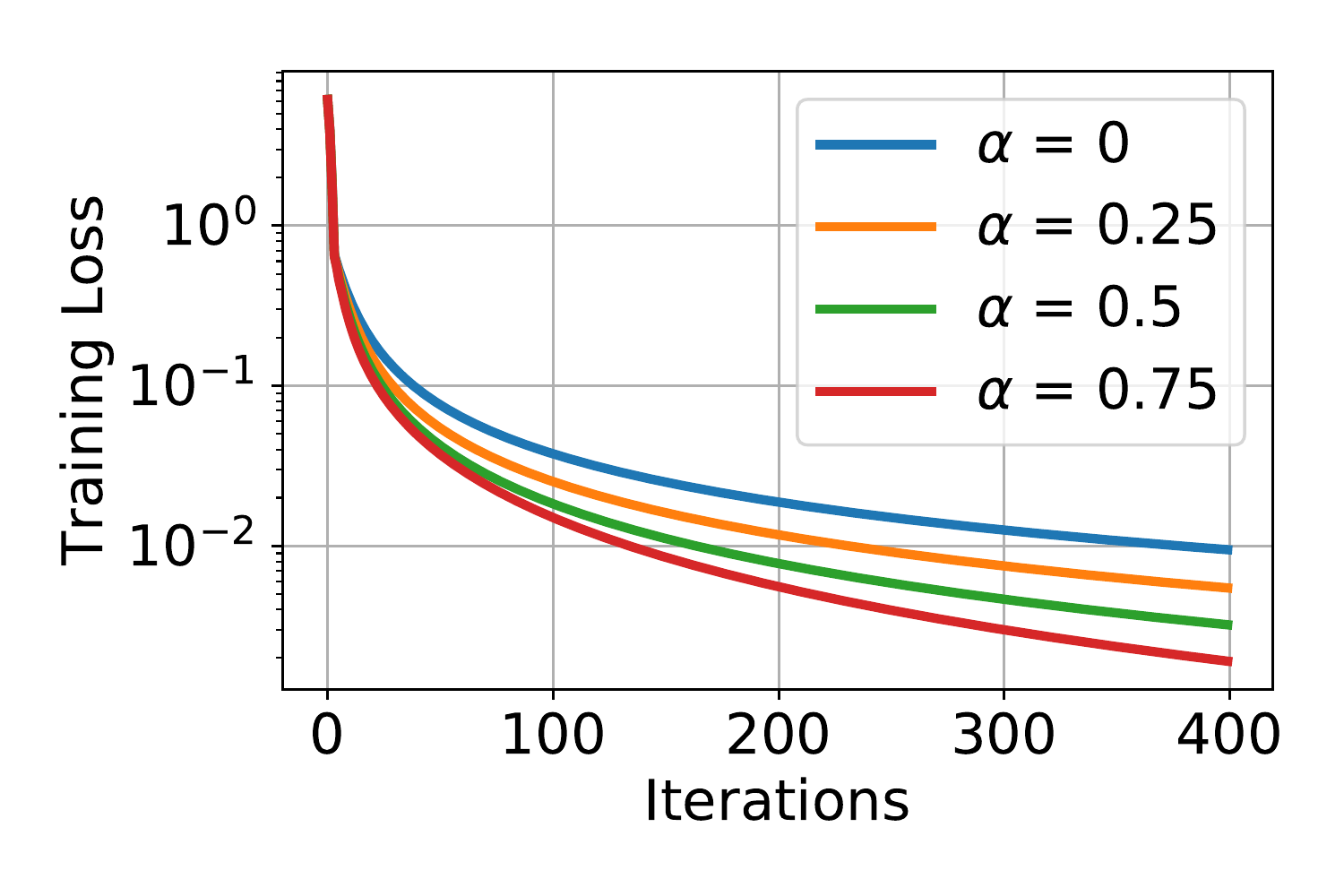}
            \caption{Training Loss}
        \end{subfigure}
        ~ 
        \begin{subfigure}[b]{0.32\textwidth}
            \includegraphics[width=\linewidth]{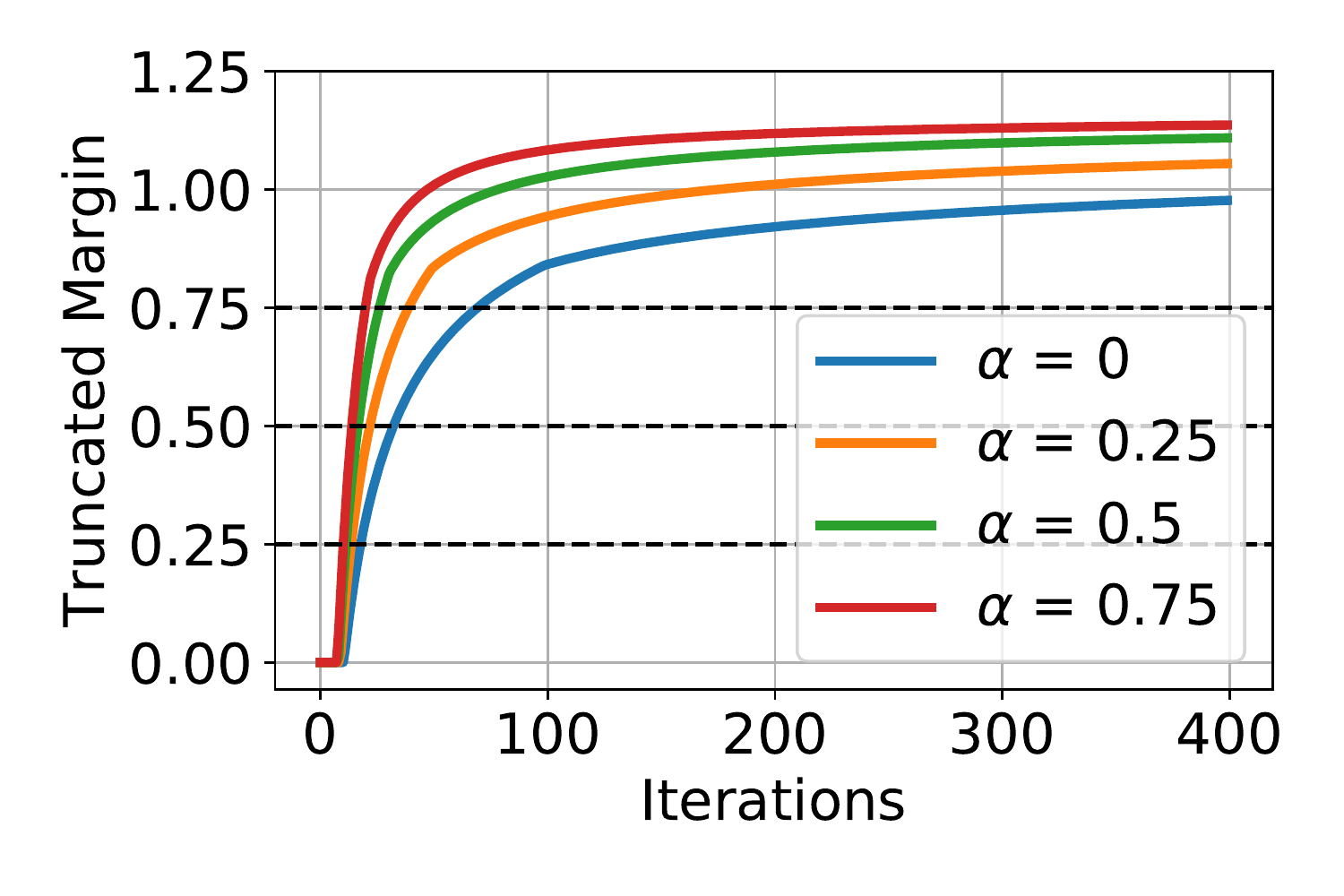}
            \caption{Truncated Margin}
        \end{subfigure}
        ~
        \begin{subfigure}[b]{0.32\textwidth}
            \includegraphics[width=\linewidth]{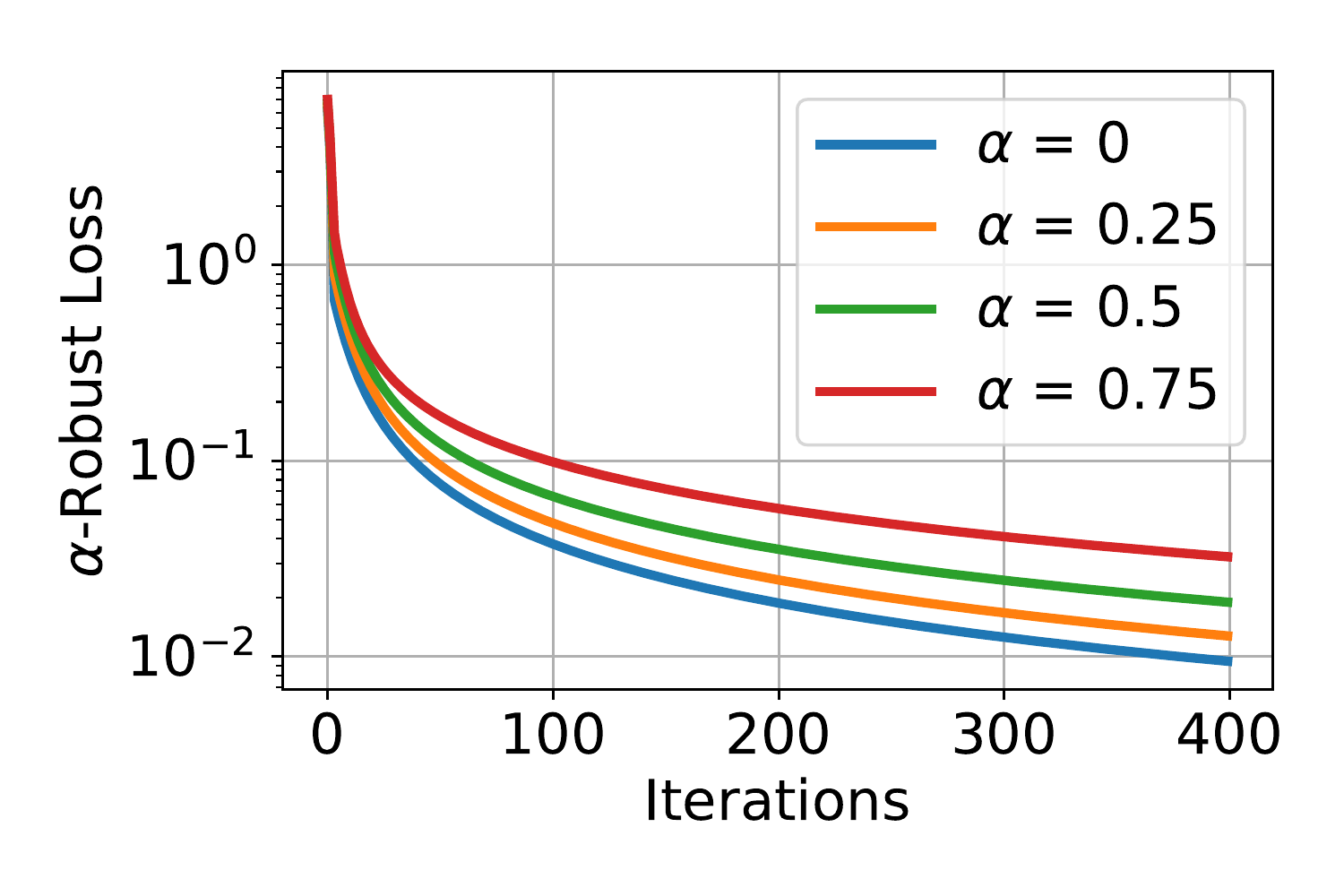}
            \caption{$\alpha$-Robust Loss}
        \end{subfigure}
        \caption{Results for $\alpha$-GD on the Iris dataset.}
        \label{fig:iris_gd}
    \end{figure}

\begin{figure}[H]
    \centering
    \begin{subfigure}[b]{0.32\textwidth}
        \includegraphics[width=\linewidth]{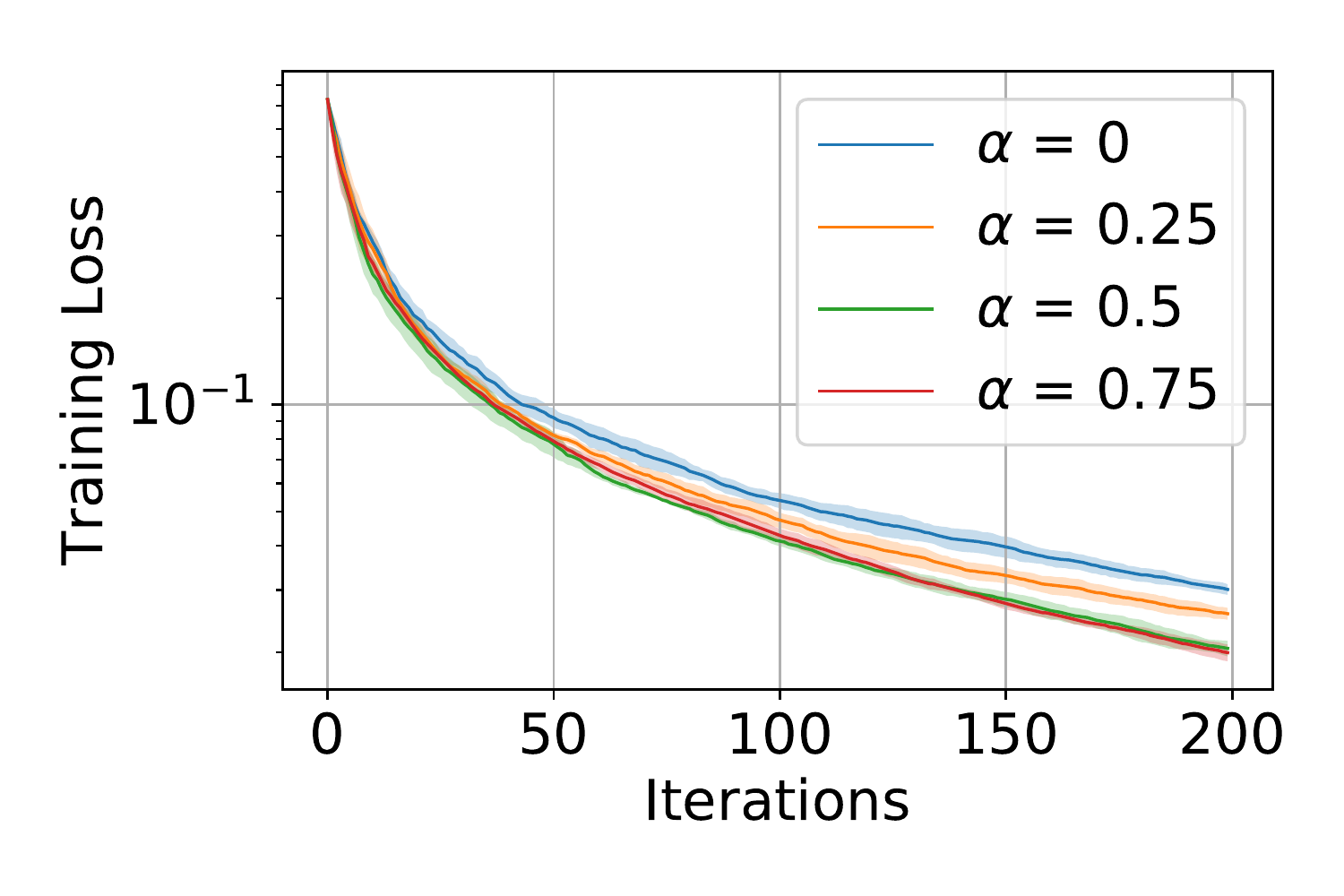}
        \caption{Training Loss}
    \end{subfigure}
    ~ 
    \begin{subfigure}[b]{0.32\textwidth}
        \includegraphics[width=\linewidth]{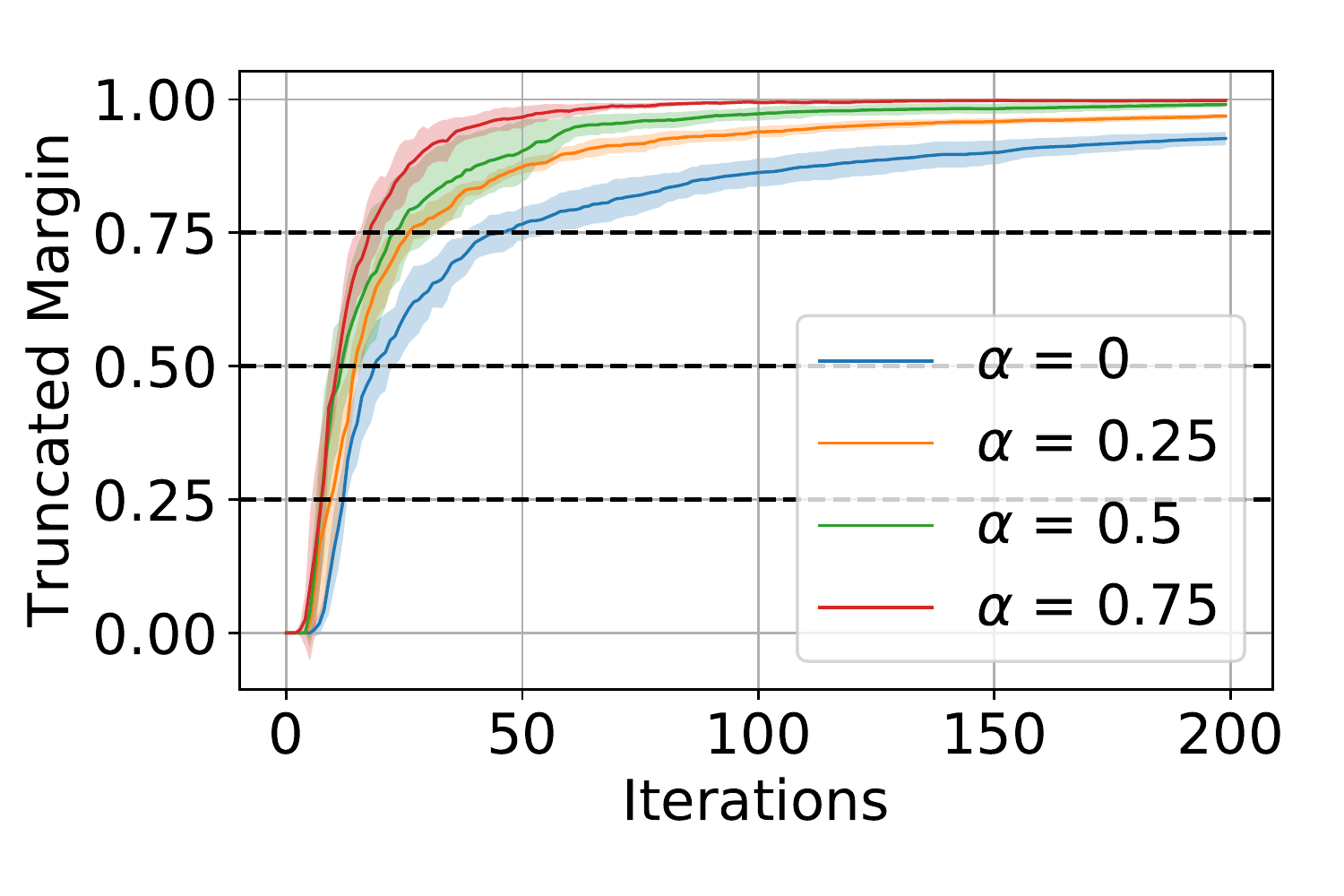}
        \caption{Truncated Margin}
    \end{subfigure}
    ~
    \begin{subfigure}[b]{0.32\textwidth}
        \includegraphics[width=\linewidth]{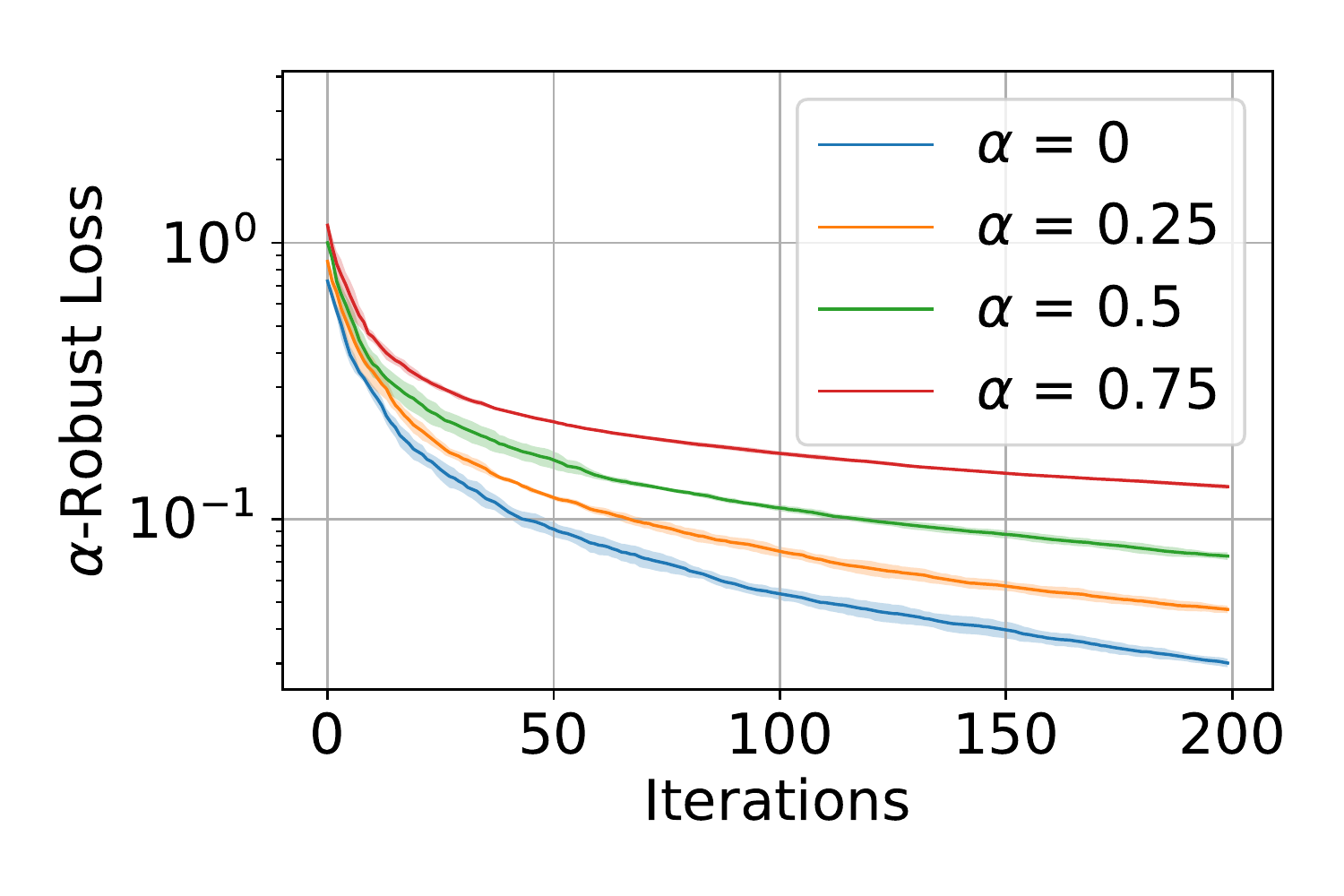}
        \caption{$\alpha$-Robust Loss}
    \end{subfigure}
    \caption{Results for $\alpha$-SGD on the synthetic dataset.}
    \label{fig:synth_sgd}
\end{figure}

\begin{figure}[H]
    \centering
    \begin{subfigure}[b]{0.32\textwidth}
        \includegraphics[width=\linewidth]{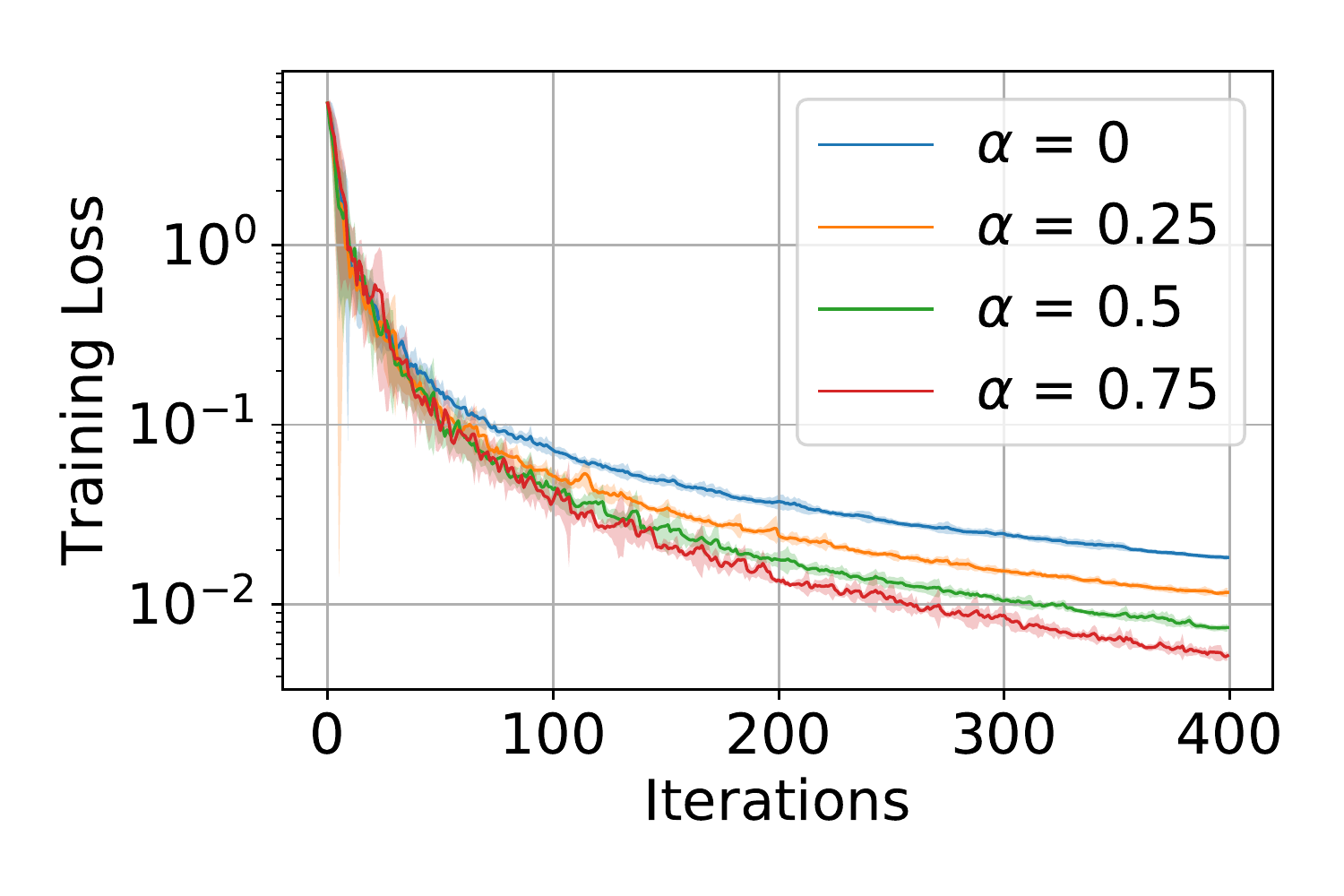}
        \caption{Training Loss}
    \end{subfigure}
    ~ 
    \begin{subfigure}[b]{0.32\textwidth}
        \includegraphics[width=\linewidth]{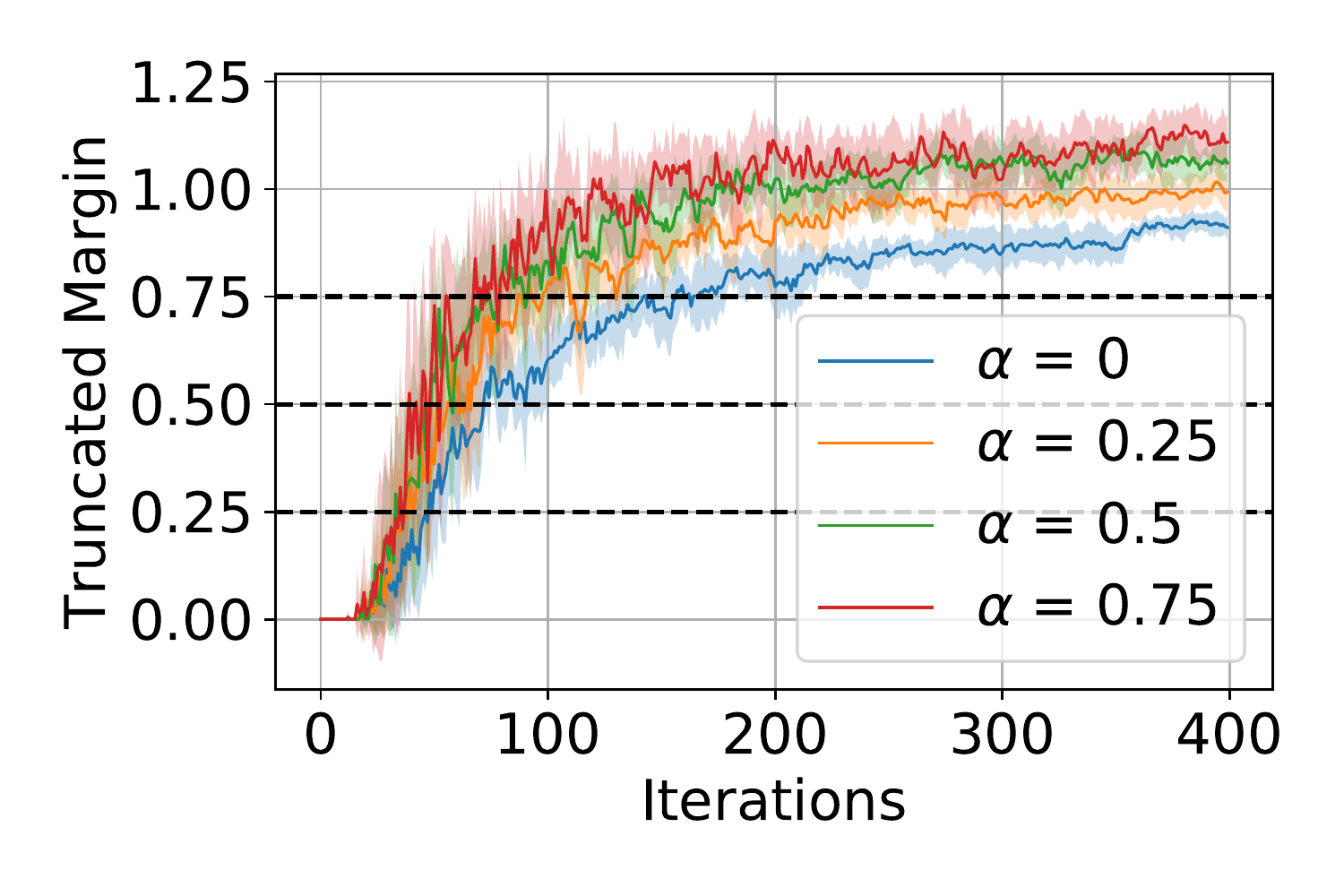}
        \caption{Truncated Margin}
    \end{subfigure}
    ~
    \begin{subfigure}[b]{0.32\textwidth}
        \includegraphics[width=\linewidth]{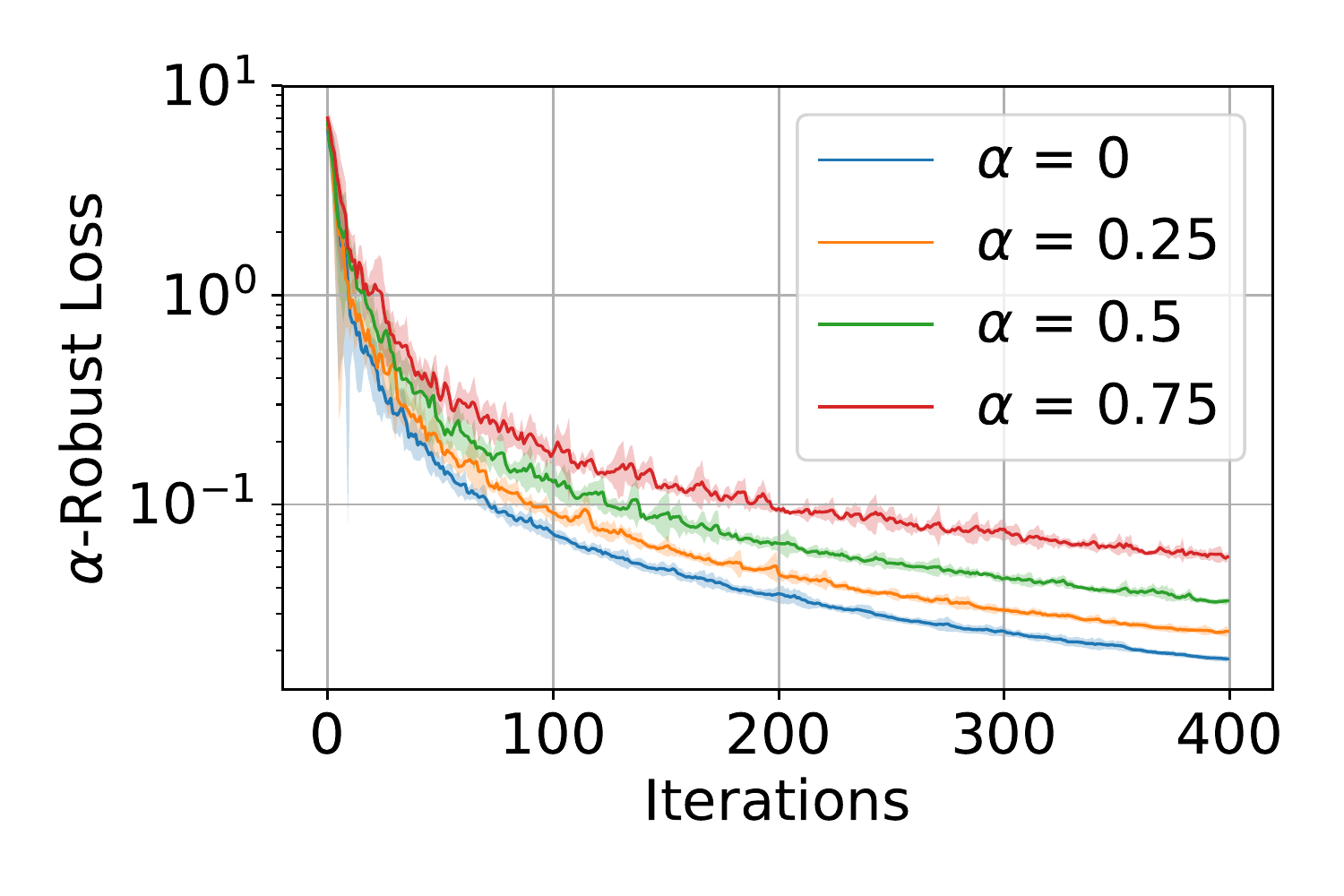}
        \caption{$\alpha$-Robust Loss}
    \end{subfigure}
    \caption{Results for $\alpha$-SGD on the Iris dataset.}
    \label{fig:iris_sgd}
\end{figure}

\paragraph*{Discussion.}
    
    The results for $\alpha$-GD on the synthetic dataset and the Iris dataset are given in Figures \ref{fig:synth_gd} and \ref{fig:iris_gd}, while the results for $\alpha$-SGD on the synthetic dataset and the Iris dataset are given in Figures \ref{fig:synth_sgd} and \ref{fig:iris_sgd}. The plots corroborate our theory for $\alpha$-GD and $\alpha$-SGD. Moreover, the results for these two methods are extremely similar on both datasets. The most notable difference is that for the margin plot on the Iris dataset, the margin for $\alpha$-SGD resembles a noisy version of the margin plot for $\alpha$-GD. This is expected, as $\alpha$-SGD focuses only on one example at a time, potentially decreasing the margin at other points, while $\alpha$-GD computes adversarial examples for every element of the training set at each iteration.

    We see that $\alpha$-GD and $\alpha$-SGD quickly attain margin $\alpha$ on both datasets, and once they do their margin convergence slows down. Moreover, the larger $\alpha$ is, generally the larger the achieved margin is at any given iteration. Generally GD and SGD take much longer to obtain a given margin than $\alpha$-GD and $\alpha$-SGD. As reflected by previous work on the implicit bias of such methods \cite{gunasekar2018characterizing, gunasekar2018implicit, nacson2019convergence, nacson2018stochastic, soudry2018implicit}, we see a logarithmic convergence to the max-margin in both settings. One interesting observation is that $\alpha$-GD and $\alpha$-SGD minimize the training loss faster than standard GD and SGD, despite not directly optimizing this loss function. Finally, we see that for $\alpha \in \{0,0.25,0.5,0.75\}$, $\alpha$-GD and $\alpha$-SGD generally seem to exhibit a $\tmO(1/t)$ convergence rate for $L_{\alpha}$. However, the convergence rate seems to increase proportionally to $\alpha$. Intuitively, $L_{\alpha}$ becomes more difficult to minimize as $\alpha$ increases.

\section{Conclusion}

	In this paper, we analyzed adversarial training on separable data. We showed that while generic adversarial training and standard gradient-based methods may each require exponentially many iterations to obtain large margin, their combination exhibits a strong bias towards models with large margin that translates to fast convergence to these robust solutions. There are a large number of possible extensions. First, we would like to understand the behavior of these methods on non-separable data, especially with regard to $L_{\rob}$. Second, we would like to generalize our results to 1) multi-class classification, and 2) regression tasks. While the former is relatively straightforward, the latter will necessarily require new methods and perspectives, due to differences in the behavior of $\ell_{\rob}$ when $\ell$ is a loss function for classification or regression.

\bibliographystyle{unsrt}
\bibliography{adv_learning_arxiv}

\newpage

\begin{appendix}

	\section{Proof of Theorem \ref{thm:adv_main}}\label{sec:fund_proof}

	Recall that in Algorithm \ref{alg:adv_train_2}, at each iteration $t$ the learner selects $S_t \subseteq S$ and then computes the adversarial examples in \eqref{eq:adv_train} for each $(x,y) \in S_t$ at the current model $w_t$. This set of adversarial examples is defined as $S'_t$. We will assume throughout that $S_t = S$, as this only diminishes the adversary's ability to obtain small margin.

	Define $S'_{< t} = \cup_{i=0}^{t-1}~S'_i$. Let $\mS^{d-1}$ denote the unit sphere in $\real^d$. For any $\epsilon \in \real$, we define $\mC(d,\epsilon)$ to be the collection of subsets of $\mS^{d-1}$ of maximal size such that any two distinct elements $w,v$ satisfy $\langle w,v \rangle < \theta$; these subsets are referred to as {\it spherical codes}. We let $N(d,\theta)$ denote the size of any $C \in \mC(d,\theta)$. For $\epsilon \leq \alpha$, we will relate the number of times an adversary can find a classifier with margin $\epsilon$ to $N(d,\epsilon/\alpha)$. In the following, we will let $e_1 \in \real^d$ be the vector with first coordinate of $1$, and remaining coordinates of $0$. Without loss of generality, we can assume the unit vector $v$ in the statement of Theorem \ref{thm:adv_main} satisfies $v = e_1$.

	\begin{lemma}\label{lem:adv1}
		Let $S = \{(\gamma e_1,1), (-\gamma e_1,-1)\}$. For any $\epsilon \leq \alpha$, there is an admissible sequence $\{w_t\}_{t \geq 0}$ such that $\margin_S(w_t) \leq \epsilon$ for all $t$ satisfying
		$$t \leq N\left(d-1,\frac{\epsilon(\gamma^2-\epsilon\alpha)}{\alpha(\gamma^2-\epsilon^2)}\right).$$
	\end{lemma}

	\begin{proof}
		Let $x_1 = \gamma e_1, x_2 = -\gamma e_1 \in \real^d$. Note that $S$ has max-margin $\gamma$.
		Fix $\epsilon \leq \alpha < \gamma$ and let
		$$\{v_1,\ldots,v_m\} \in \mC\left(d-1,\frac{\epsilon(\gamma^2-\epsilon\alpha)}{\alpha(\gamma^2-\epsilon^2)}\right).$$
		Let $a = \epsilon/\gamma$. For $1 \leq t \leq m$, define $w_t$ by
		$$w_t^T = [a~~(\sqrt{1-a^2})v_t^T].$$
		That is, the first coordinate of $w_t$ is $a$, while its remaining $d-1$ coordinates are given by $\sqrt{1-a^2}v_t$. Since $\|v_t\| = 1$, we have $\|w_t\| = 1$. We will show that each $w_t$ is admissible and has margin at most $\epsilon$ with respect to $S$.

		For any $t$, we have
		\begin{align*}
		\langle w_t, x_1 \rangle = \gamma a = \epsilon > 0\\
		\langle w_t, x_2 \rangle = -\gamma a = -\epsilon < 0.\end{align*}

		Thus, each $w_t$ correctly classifies $S$. Moreover, since $\|w_t\| = 1$, its margin at $S$ is $\epsilon$. We now must show that each $w_t$ correctly classifies $S'_{< t}$. 

		Recall that we assume $\ell(w,x,y)$ is of the form $f(-y\langle w,x \rangle)$ where $f$ is a monotonically increasing function. This implies that given $w$, $\alpha > 0$, and $(x,y)$, $\delta = -y\alpha \frac{w}{\|w\|}$ satisfies \eqref{eq:adv_train}. Therefore, for $t \geq 0$,
		\begin{align*}
		S'_{< t} &= \bigcup_{j=0}^{t-1}\left\{\left(x_i-\alpha y_i \frac{w_j}{\|w_j\|}, y_i\right)\right\}_{i=1}^2\\
		&= \{ (\gamma e_1-\alpha w_i, 1)~|~0 \leq i \leq t-1\} \cup \{ (-\gamma e_1 + \alpha w_j, -1)~|~ 0 \leq j \leq t-1\}.\end{align*}

		Given $t \geq 1$ and $i < t$, and by construction of the $v_i$, we have
		\begin{align*}
		\langle w_t, \gamma e_1 -\alpha w_i\rangle &= \langle w_t, \gamma e_1 \rangle - \alpha \langle w_t, w_i\rangle\\
		&= \epsilon -\alpha(a^2 + (1-a)^2\langle v_t, v_i\rangle)\\
		&= \epsilon - \dfrac{\alpha\epsilon^2}{\gamma^2} -\alpha\left(1-\frac{\epsilon^2}{\gamma^2}\right)\langle v_t,v_i\rangle\\
		& > \epsilon - \dfrac{\alpha\epsilon^2}{\gamma^2} -\alpha \left(1-\frac{\epsilon^2}{\gamma^2}\right)\frac{\epsilon(\gamma^2-\epsilon\alpha)}{\alpha(\gamma^2-\epsilon^2)}\\
		& = 0.\end{align*}

		An analogous computation shows that $\langle w_t, -\gamma e_1 + \alpha w_i\rangle < 0$. Thus, $w_t$ linearly separates $S'_{<t}$, and has margin $\epsilon$ at $S$, proving the desired result.
	\end{proof}

	While finding exact values of $N(d,\epsilon)$ is difficult \cite{cohn2014sphere}, there are straightforward lower bounds. In particular, we have the following lemma.

	\begin{lemma}\label{lem:adv2}
	Let $d > 1, 0 < \epsilon < 1$. There is some constant $c$ such that $N(d,\epsilon) \geq \frac{1}{2}\exp(cd\epsilon^2)$.
	\end{lemma}

	\begin{proof}
		Fix some integer $q \geq k$, and and let $\{v_i\}_{i=1}^q$ be an orthonormal basis of $\real^q$. By the distributional Johnson-Lindenstrauss lemma, there is some distribution $\mD$ over $\real^{d \times q}$ such that for all $x \in \real^q$ and $A \sim \mD$,
		\begin{align*}
		\PP\left(| \|Ax\|^2 - \|x\|^2 | > \tau \right) &\leq 2\exp(-(\tau^2-\tau^3)d/4)).\end{align*}

		Setting $\tau = \epsilon/2$ and taking a union bound over the $q^2$ vectors of the form $v_i, v_i + v_j, v_i - v_j$ (for $ i \neq j$), this implies that there is a universal constant $c$ such that if $q \leq \frac{1}{2}\exp(cd\epsilon^2)$, then there is some $A \in \real^{d \times k}$ such that for $1 \leq i, j \leq q, i \neq j$,
		$$(1-\tau)\|v_i\|^2 \leq \|Av_i\|^2 \leq (1+\tau)\|v_i\|^2$$
		$$(1-\tau)\|v_i \pm v_j\|^2 \leq \|A(v_i\pm v_j)\|^2 \leq (1+\tau)\|v_i \pm v_j\|^2.$$

		Taking this $A$ and letting $w_i = Av_i/\|Av_i\|$, we have that for $i \neq j$,
		\begin{align*}
		\langle w_i, w_j\rangle &= \dfrac{\|A(v_i+v_j)\|^2 - \|A(v_i-v_j)\|^2}{4\|Av_i\|\|Av_j\|}\\
		&\leq \dfrac{(1+\tau)\|v_i+v_j\|^2 - (1-\tau)\|v_i-v_j\|^2}{4(1-\tau)}\\
		&= \frac{\tau}{1-\tau}\\
		&\leq \epsilon.\end{align*}

		Here we used the fact that the $v_i$ are orthonormal and that $\tau \leq \epsilon/2$. Hence, the $q$ vectors $\{w_i\}_{i=1}^q$ are all unit vectors such that for $i \neq j$, $\langle w_i, w_j\rangle \leq \epsilon$.
	\end{proof}	

	Theorem \ref{thm:adv_main} then follows directly by combining Lemmas \ref{lem:adv1} and Lemma \ref{lem:adv2} and using the fact that $\alpha < \gamma$.

\section{Proof of Lemma \ref{lem:linear_robust_loss}}\label{sec:main_lem_proof}

	\begin{proof}[Proof of (1)]
		Fix $w, x \in \real^d$ and $y \in \{\pm 1\}$. Suppose $\|\delta\| \leq \alpha$. Since $f$ is monotonically increasing,
		\begin{align*}
			\ell(w,x+\delta,y) &= f(-y\langle w, x+\delta \rangle)\\
			&= f(-y\langle w, x\rangle + \langle w, -y\delta \rangle)\\
			&\leq f(-y\langle w,x \rangle + \|w\|\|-y\delta\|)\\
			&= f(-y\langle w,x\rangle + \alpha\|w\|).
		\end{align*}
		Taking a supremum over both sides, we derive the desired result.
	\end{proof}

	\begin{proof}[Proof of (2)]
		Since $f$ is differentiable, $\ell(w,x,y) = f(-y\langle w,x\rangle)$ is differentiable. By Proposition \ref{prop:danskin}, $\ell_{\rob}(w,x,y)$ is subdifferentiable. By (1), we find
		$$\ell(w,x-y\alpha \overline{w},y) = f(-y\langle w, x \rangle + \alpha\|w\|) = \ell_{\rob}(w,x,y).$$
		Therefore, letting $\delta' = -y\alpha\overline{w}$, we find
		$$\delta'\in \argmax_{\|\delta\|\leq \alpha}\ell(w,x+\delta,y).$$
		By Proposition \ref{prop:danskin}, this implies $\nabla \ell(w,x+\delta',y) \in \partial \ell_{\rob}(w,x,y)$ where the gradient is taken by treating $\delta'$ as constant w.r.t. $w$. By direct computation,
		$$\nabla \ell(w,x+\delta',y) = f'(-y\langle w,x\rangle + \alpha\|w\|)(-yx + \alpha\overline{w}).$$
	\end{proof}

	\begin{proof}[Proof of (3)]
		Suppose $w \neq 0$. Note that by the Cauchy-Schwarz inequality, if $\|\delta\| \leq \alpha$, then $\langle w, -y\delta\rangle \leq \alpha\|w\|$ with equality if and only if $\delta = -y\alpha\overline{w}$. Since $f$ is strictly increasing, if $\delta \neq -y\alpha\overline{w}$ then
		$$\ell(w,x+\delta,y) = f(-y\langle w, x\rangle + \langle w,-y\delta\rangle) < f(-y\langle w,x\rangle + \alpha\|w\|).$$
		Therefore, $\delta' = -y\alpha\overline{w}$ is the unique maximizer of $\ell(w,x+\delta,y)$ subject to $\|\delta\|\leq \alpha$. By Proposition \ref{prop:danskin}, this implies that $\ell_{\rob}(w,x,y)$ is differentiable at this point with gradient as in (2).
	\end{proof}

	\begin{proof}[Proof of (4)]
		Suppose $w \neq 0$ and define $z = -yx + \alpha \overline{w}$. By (3), $\ell_{\rob}(w,x,y)$ is differentiable with gradient given by
		$$\nabla_w \ell_{\rob}(w,x,y) = f'(\langle w,z\rangle)v$$
		where $z$ is treated as constant with respect to $w$. By elementary calculus,
		\begin{align*}
		\nabla_w^2 \ell_{\rob}(w,x,y) &= \dfrac{\alpha f'(\langle w,z\rangle)}{\|w\|^2}\left(\|w\|I - \frac{ww^T}{\|w\|}\right) + f''(\langle w,z\rangle)zz^T.\\
		&= \dfrac{\alpha f'(\langle w, z\rangle)}{\|w\|}I - \dfrac{f'(\langle w, z\rangle)}{\|w\|^3}ww^T + f''(\langle w,z \rangle)zz^T.
		\end{align*}

		Define the following matrices:
		\begin{align*}
		A &= \dfrac{\alpha f'(\langle w, z\rangle)}{\|w\|}I\\
		B &= - \dfrac{\alpha f'(\langle w, z\rangle)}{\|w\|^3}ww^T\\
		C &= f''(\langle w,z \rangle)zz^T.\end{align*}

		Given a real symmetric matrix $X$, let $\lambda_1(X)$ denote its largest eigenvalue. Given $q \in \real^d$, note that $\lambda_1(qq^T) = \|q\|^2$, while its remaining eigenvalues are 0. Therefore,
		\begin{align*}
		\lambda_1(A) & = \dfrac{\alpha f'(\langle w, z\rangle)}{\|w\|} \leq \dfrac{\alpha M}{\|w\|}\\
		\lambda_1(B) & \leq 0\\
		\lambda_1(C) &= f''(\langle w,z \rangle)\|z\|^2 \leq \beta(\|x\|+\alpha)^2.\end{align*}

		For $\lambda_1(A)$,, we used the fact that $f$ is $M$-Lipschitz, while for $\lambda_1(C)$, we used the fact that $f$ is $\beta$-smooth and that $\|z\| \leq \|z\| + \alpha$. By the interleaving property of eigenvalues for Hermitian matrices, this implies
		\begin{align*}
		\lambda_1(\nabla_w^2 \ell_{\rob}(w,x,y)) &\leq \lambda_1(A) + \lambda_1(B) + \lambda_1(C)\\
		&\leq \dfrac{\alpha M}{\|w\|} + \beta(\|x\|+\alpha)^2.\end{align*}
	\end{proof}

	\begin{proof}[Proof of (5)]
		This follows directly from the fact that a supremum of convex functions is convex and $\ell_{\rob}(w,x,y)$ can be written as a supremum of functions of the form $\ell(w,x+\delta,y)$.
	\end{proof}

	\section{Proof of Results in Section \ref{sec:gd_log_reg}}\label{sec:gd_proof}

	\subsection{Proof of Theorem \ref{thm:gd_log_reg1}}

		First, we show some form of smoothness holds in straight-line segments between the iterates $\{w_t\}_{t \geq 1}$.

		\begin{lemma}\label{lem:ell_rob_smooth}
			Suppose $w_0 = 0, \eta_0 = 0$. For all $t \geq 1$, $(x_i,y_i) \in S$, and $v \in \conv(w_t,w_{t+1})$, $\ell_{\rob}(w,x_i,y_i)$ is twice differentiable at $v$ and $\nabla^2 \ell_{\rob}(v,x_i,y_i) \preceq \beta'I$ where
			\begin{equation}\label{eq:beta_prime}
			\beta' = \dfrac{2\alpha}{\gamma-\alpha} + (1+\alpha)^2.\end{equation}
		\end{lemma}

		\begin{proof}
			Fix $t \geq 1$. Let $w^*$ be a unit-norm max-margin classifier. We first show that $\langle w_t,w^*\rangle$ is bounded below. Given $j \geq 0$ and $(x_i,y_i) \in S$, define $\z{j}_i = -y_ix_i$ if $w_j = 0$ and otherwise
			$$\z{j}_i = -y_ix_i + \alpha\dfrac{w_j}{\|w_j\|}.$$
			Then by Lemma \ref{lem:linear_robust_loss},
			$$\ell_{\rob}(w_j,x_i,y_i) = f(\langle w_j, \z{j}_i \rangle)$$
			$$\nabla \ell_{\rob}(w_j,x_i,y_i) = f'(\langle w_j, \z{j}_i \rangle)\z{j}_i.$$

			By Cauchy-Schwarz,
			$$\langle \z{j}_i, w^*\rangle = -y\langle x_i, w^*\rangle + \alpha  \geq -\gamma + \alpha.$$

			Since $w_0 = 0$, the update of $\alpha$-gradient descent implies
			\begin{align*}
			w_t =  - \sum_{j < t} \eta_j\nabla L_{\rob}(w_j) = -\sum_{j < t} \dfrac{\eta_j}{n} \sum_{i=1}^n f'(\langle w_j, \z{j}_i\rangle)\z{j}_i.\end{align*}

			Therefore, for any $t \geq 1$,
			\begin{align*}
			\langle w_t,w^*\rangle &= -\sum_{j < t} \dfrac{\eta_j}{n} \sum_{i=1}^n f'(\langle w_j, \z{j}_i\rangle)\langle \z{j}_i,w^*\rangle\\
			& \geq -\sum_{j < t} \dfrac{\eta_j}{n} \sum_{i=1}^n f'(\langle w_j,\z{j}_i\rangle)(-\gamma+\alpha)\\
			& = (\gamma-\alpha)\sum_{j < t} \dfrac{\eta_j}{n} \sum_{i=1}^n f'(\langle w_j,\z{j}_i\rangle)\\
			& \geq (\gamma-\alpha) \dfrac{\eta_0}{n} \sum_{i=1}^n f'(\langle w_0, \z{0}_i\rangle)\\
			& = (\gamma-\alpha)\eta f'(0)\\
			&= \frac{\gamma-\alpha}{2}.
			\end{align*}
			Here we used the fact that $f$ is strictly increasing, so $f'(a) > 0$ for all $a$, as well as the fact that $f'(0) = 1/2$.

			Let $v \in \conv(w_t,w_{t+1})$. By convexity of $\langle \cdot, w^*\rangle$ and the Cauchy-Schwarz inequality,
			$$\frac{(\gamma-\alpha)}{2} \leq \langle v,w^*\rangle \leq \|v\|.$$

			Let $(x_i,y_i) \in S$. By Lemma \ref{lem:linear_robust_loss}(4), this implies that for all $v \in \conv(\{w_t,w_{t+1}\})$, $\ell(w,x_i,y_i)$ is twice differentiable at $w = v$ and satisfies $\nabla^2 \ell_{\rob}(v,x,y) \preceq \beta'I$ where
			$$\beta' = \dfrac{\alpha}{\|v\|} + (1+\alpha)^2 \leq \dfrac{2\alpha}{\gamma-\alpha} + (1+\alpha)^2.$$
			Here we used the fact that $f$ is $1$-Lipschitz and $1$-smooth, and that $\|x_i\| \leq 1$ by assumption, and then combined this with out lower bound on $\|v\|$ from above.

			Therefore, if $v \in \conv(w_t,w_{t+1})$ for $t \geq 1$ then
			$$L_{\rob}(v) = \frac{1}{n}\sum_{i=1}^n \ell_{\rob}(v,x_i,y_i).$$
			Therefore, $L_{\rob}$ is the average of $n$ functions that are twice-differentiable at $v$, and is therefore itself twice-differentiable at $v$. By basic properties of Hermitian matrices, we have
			$$\nabla^2 L_{\rob}(v) = \frac{1}{n}\sum_{i=1}^n \nabla^2\ell_{\rob}(v,x_i,y_i) \preceq \beta' I.$$
		\end{proof}		
		Therefore, $L_{\rob}(w)$ is also twice differentiable at such $v$ and satisfies $\nabla^2 L_{\rob}(w) \preceq \beta'I$.	Using this, we then derive the following bound on the difference between $L_{\rob}$ at iterates $w_{t+1}, w_t$.

		\begin{lemma}\label{lem:smooth_iters}
			Suppose $\{w_t\}_{t \geq 0}$ are the iterates of $\alpha$-GD on $L$ with $w_0 = 0, \eta_0 = 1$, and constant step-size $\eta < 2(\beta')^{-1}$ for $t \geq 1$ where
			$$\beta' = \dfrac{2\alpha}{\gamma-\alpha} + (1+\alpha)^2.$$
			Then for $t \geq 1$, 
			$$L_{\rob}(w_{t+1}) \leq L_{\rob}(w_t) - \eta\left(1-\dfrac{\eta\beta'}{2}\right)\|\nabla L_{\rob}(w_t)\|^2.$$
		\end{lemma}

		\begin{proof}
			Let $t \geq 1$. By Lemma \ref{lem:ell_rob_smooth}, $L_{\rob}$ is twice differentiable on $\conv(w_t,w_{t+1})$. By Taylor's theorem, there is some $v \in \conv(w_t,w_{t+1})$ such that
			$$L_{\rob}(w_{t+1}) = L_{\rob}(w_t) + \langle \nabla L_{\rob}(w_t), w_{t+1}-w_t \rangle+ \dfrac{(w_{t+1}-w_t)^T\nabla^2 L_{\rob}(v) (w_{t+1}-w_t)}{2}.$$
			By Lemma \ref{lem:ell_rob_smooth}, $\nabla^2 L_{\rob}(v) \preceq \beta'I$. Therefore,
			\begin{align*}
			L_{\rob}(w_{t+1}) &\leq L_{\rob}(w_t) + \langle \nabla L_{\rob}(w_t), w_{t+1}-w_t \rangle + \dfrac{\beta'}{2}\|w_{t+1}-w_t\|^2\\
			&= L_{\rob}(w_t) - \eta\|\nabla L_{\rob}(w_t)\|^2 + \dfrac{\eta^2\beta'}{2}\|\nabla L_{\rob}(w_t)\|^2\\
			&= L_{\rob}(w_t) - \eta\left(1-\dfrac{\eta\beta'}{2}\right)\|\nabla L_{\rob}(w_t)\|^2.
			\end{align*}
		\end{proof}

		Next, we introduce a lemma about smooth, convex functions. We specifically use the version from \cite{ji2018risk}, Lemma 3.3.

		\begin{lemma}\label{lem:smooth_convex}
		Suppose $h$ is convex and there exists $\beta \geq 1$ and $\{\eta_t\}_{t \geq 1}$ such that $\eta_t\beta \leq 1$ for all $t \geq 1$ the gradient descent iterates $\{w_t\}_{t \geq 1}$ defined by $w_{t+1} = w_t - \eta_t\nabla h(w_t)$ satisfy
		$$h(w_{t+1}) \leq h(w_t) - \eta_t\left(1-\frac{\eta_t\beta}{2}\right)\|\nabla h(w_j)\|^2.$$
		Then for any $w \in \real^d$,
		$$\left(2\sum_{j= 1}^{t-1} \eta_j\right)(h(w_t)-h(w)) \leq \|w_1-w\|^2 - \|w_t-w\|^2.$$
		\end{lemma}

		We can now prove the desired main theorem.

		\begin{proof}[Proof of Theorem \ref{thm:gd_log_reg1}]
			Define
			$$\beta' = \dfrac{2\alpha}{(\gamma-\alpha)} + (1+\alpha)^2.$$
			Recall that by assumption we have
			\begin{equation}\label{eq:gd_eta_eq}
			\eta_t \leq \left(\dfrac{2\alpha}{(\gamma-\alpha)} + (1+\alpha)^2\right)^{-1}.\end{equation}
			Therefore, $\eta_t\beta' \leq 1$ holds for $t \geq 1$. By Lemma \ref{lem:smooth_iters}, we know that for $t \geq 1$, we have
			$$L_{\rob}(w_{t+1}) \leq L_{\rob}(w_t) - \eta\left(1-\dfrac{\eta\beta'}{2}\right)\|\nabla L_{\rob}(w_t)\|^2.$$
			By Lemma \ref{lem:smooth_convex}, this implies that for any $w \in \real^d$,
			\begin{equation}\label{eq:gd_log_reg1_0}
			\left(2\sum_{j= 1}^{t-1} \eta_j\right)(L_{\rob}(w_t)-L_{\rob}(w)) \leq \|w_1-w\|^2 - \|w_t-w\|^2.\end{equation}
			Define $u_t = \dfrac{\ln(t)}{\gamma-\alpha}w^*$ where $w^*$ is a unit vector achieving margin $\gamma$ (this exists by assumption on $S$). That is, for all $(x_i,y_i) \in S$,
			\begin{equation}\label{eq:gd_log_reg1_1}
			y_i\langle w^*, x_i \rangle \geq \gamma.\end{equation}
			By direct computation,
			\begin{equation}\label{eq:gd_log_reg1_2}
			\|u_t\| = \dfrac{\ln(t)^2}{\gamma-\alpha}\end{equation}
			and by Lemma \ref{lem:linear_robust_loss} we have
			\begin{equation}\label{eq:gd_log_reg1_3}
			\begin{split}
			L_{\rob}(u_t) &= \frac{1}{n}\sum_{i=1}^n \ell_{\rob}(u_t,x_i,y_i)\\
			&= \frac{1}{n}\sum_{i=1}^n f(-y_i\langle u_t,x_i\rangle + \alpha\|u_t\|)\\
			&= \frac{1}{n}\sum_{i=1}^n f(\|u_t\| (\langle w^*,-y_ix_i\rangle + \alpha)
			\end{split}
			\end{equation}

			By \ref{eq:gd_log_reg1_1} and \label{eq:gd_log_reg1_2}, for all $i$,
			\begin{align*}
			\|u_t\| (\langle w^*,-y_ix_i\rangle + \alpha) &= \frac{\ln(t)}{\gamma-\alpha}(\langle w^*,-y_ix_i\rangle + \alpha)\\
			&\leq\frac{\ln(t)}{\gamma-\alpha}(-\gamma+\alpha)\\
			&\leq -\ln(t).\end{align*}

			Since $f$ is monotonically increasing, by \eqref{eq:gd_log_reg1_3}, we have
			\begin{align*}
			L_{\rob}(u_t) \leq \frac{1}{n}\sum_{i=1}^nf(-\ln(t)) \leq \ln(1+\exp(-\ln(t))) \leq \frac{1}{t}.\end{align*}
			Here we used the fact that for all $x > 0$, $\ln(1+x) \leq x$. Rearranging \eqref{eq:gd_log_reg1_0}, and using the fact that $(a+b)^2 \leq 2a^2 + 2b^2$, we have
			\begin{align*}
			L_{\rob}(w_t) &\leq L_{\rob}(u_t) + \dfrac{\|w_1-u_t\|^2}{2\sum_{j= 1}^{t-1} \eta_j}\\
			&\leq \frac{1}{t} + \dfrac{\|w_1\|^2 + \|u_t\|^2}{\sum_{j= 1}^{t-1} \eta_j}\\
			&\leq \frac{1}{t} + \dfrac{\|w_1\|^2 + \frac{\ln(t)^2}{(\gamma-\alpha)^2}}{\sum_{j=1}^{t-1}\eta_j}.
			\end{align*}

			It suffices to bound $\|w_1\|$. Since $w_0 = 0, \eta_0 = 1$, we have
			\begin{align*}
			w_1 &= \frac{-1}{n}\sum_{i=1}^n \nabla \ell_{\rob}(w_0,x_i,y_i)\\
			&= \frac{-1}{n}\sum_{i=1}^n \nabla f'(\langle w_0,-yx\rangle)(-y_ix_i)\\
			&= \frac{1}{2n}\sum_{i=1}^n y_ix_i.\end{align*}

			By the triangle inequality, $\|w_1\| \leq \frac{1}{2n}\sum_{i=1}^n \|x_i\| \leq \frac{1}{2}$, and so
			$$L_{\rob}(w_t) \leq \frac{1}{t} + \frac{\frac{1}{4} + \frac{\ln(t)^2}{(\gamma-\alpha)^2}}{\sum_{j=1}^{t-1}\eta_j}.$$
		\end{proof}

	\subsection{Proof of Lemma \ref{lem:log_reg_margin}}

		\begin{proof}
		Since $\ell_{\rob} \geq 0$, if $L_{\rob}(w) \leq \ln(2)/n$ then for $1 \leq i \leq n$, $\ell_{\rob}(w,x_i,y_i) \leq \ln(2)$. By Lemma \ref{lem:linear_robust_loss},
		\begin{align*}
			& \ell_{\rob}(w,x_i,y_i) \leq \ln(2)\\
			& \implies \ln(1+\exp(-\langle w,y_ix_i \rangle + \alpha\|w\|)) \leq \ln(2)\\
			& \implies -\langle w,y_ix_i \rangle + \alpha\|w\| \leq 0\\
			& \implies \dfrac{\langle w, y_ix_i\rangle}{\|w\|} \geq \alpha.
		\end{align*}\end{proof}		

	\subsection{Proof of Corollary \ref{thm:gd_log_reg2}}

		\begin{proof}
			Since $\eta \leq 1$, Theorem \ref{thm:gd_log_reg1} implies that for $t \geq 2$,
			\begin{align*}
			L_{\rob}(w_t) &\leq \frac{1}{t} + \frac{\frac{1}{4} + \frac{\ln(t)^2}{(\gamma-\alpha)^2}}{\eta(t-1)}\\
			&\leq \frac{\frac{5}{4} + \frac{\ln(t)^2}{(\gamma-\alpha)^2}}{\eta(t-1)}.\end{align*}

			Define $C_q$ by $C_q = \inf\{t \geq 2 | \frac{5}{4} + \ln(t)^2 \leq (t-1)t^{-1/q}\}$.
			Note that $C_q < \infty$ by basic logarithm properties. For $t \geq C_q$, Theorem \ref{thm:gd_log_reg1} implies
			\begin{align*}
			L_{\rob}(w_t) &\leq \frac{\frac{5}{4} + \frac{\ln(t)^2}{(\gamma-\alpha)^2}}{\eta(t-1)}\\
			&\leq \dfrac{t^{-1/q}}{\eta(\gamma-\alpha)^2}.\end{align*}

			Therefore, if $t$ satisfies \eqref{eq:gd_log_reg_margin}, then $L_{\rob}(w_t) \leq \frac{\ln(2)}{n}$. We conclude by applying Lemma \ref{lem:log_reg_margin}.
		\end{proof}	

	\subsection{Proof of Theorem \ref{thm:exp_gd}}

		\begin{proof}
			We have the following recursive formula for the gradient descent iterates:
			$$w_{t+1} = w_t - \eta \nabla \ell(w,x,y) = w_t + f'(-\langle w_t, x\rangle)x.$$
			Therefore, $w_t = (a_t,c)$ where $a_0 = 0$ and $a_t$ is given recursively by
			\begin{equation}
			a_{t+1} = a_t + f'(-a_t) = a_t + \dfrac{1}{1+\exp(a_t)}.
			\end{equation}
			We will first show inductively that $a_t \leq \ln(t+1)$. For $t = 0$, this holds trivially. Otherwise, note that $x+(1+\exp(x))^{-1}$ is a strictly increasing function. Thus, by the inductive hypothesis,
			\begin{align*}
			a_{t+1} &= a_t + (1+\exp(a_t))^{-1}\\
			&\leq \ln(t+1) + (t+2)^{-1}\\
			&\leq \ln(t+2).\end{align*}
			This last step follows from the fact that $\ln(x+1) + (x+2)^{-1} \leq \ln(x+2)$ for all $x > 0$. Suppose $\margin_S(w_t) \geq \alpha$. By the definition of margin,
			\begin{align*}
			\margin_S(w_t) = \dfrac{\langle w_t, x\rangle}{\|w_t\|} = \dfrac{a_t}{\sqrt{a_t^2+c^2}} \geq \alpha.\end{align*}
			Rearranging and using the fact that $0 < \alpha < 1$, this implies
			\begin{align*}
			a_t &\geq \dfrac{\alpha c}{\sqrt{1-\alpha^2}}\geq \dfrac{c}{1-\alpha}.\end{align*}
			Since $a_t \leq \ln(t+1)$, this implies that $t+1 \geq \exp(c/(1-\alpha))$, proving the result.
		\end{proof}		

	\section{Proof of Results in Section \ref{sec:sgd_log_reg}}\label{sec:sgd_proof}

\subsection{Proof of Theorem \ref{thm:log_reg1}}

	At each iteration $t$, $\alpha$-SGD selects $i_t$ uniformly at random from $\{1,\ldots, n\}$. Let $v_t = -y_{i_t}x_{i_t}$. Given $w_t$, define $\overline{w}_t$ to be $w_t/\|w_t\|$ if $w_t \neq 0$, and $0$ otherwise. Finally, define $z_t = -y_{i_t}x_{i_t} + \alpha \overline{w}_t$.

	By Lemma \ref{lem:linear_robust_loss},
	\begin{equation}\label{eq:z1}
	\ell_{\rob}(w_t,x_{i_t},y_{i_t}) = f(\langle w_t, z_t\rangle)\end{equation}

	By definition of $w^*$ and the Cauchy-Schwarz inequality, we have
	\begin{equation}\label{eq:z2}
	\begin{split}
	\langle w^*, z_t\rangle = \langle w^*, -y_{i_t}x_{i_t} + \alpha \overline{w}_t\rangle \leq -\gamma + \alpha.
	\end{split}
	\end{equation}
	By Lemma \ref{lem:linear_robust_loss}, the iterates of $\alpha$-SGD are given recursively by $w_{t+1} = w_t-\eta_t g_t$ where
	\begin{equation}\label{eq:z3}
	g_t = f'(\langle w_t,z_t\rangle)z_t.\end{equation}

	To prove the desired result, we will analyze the following two quantities:
	\begin{equation}\label{f_eq}
	f'_{< t} := \sum_{j < t} \eta_jf'(\langle w_j,z_j\rangle).\end{equation}
	\begin{equation}\label{F_eq}
	F'_{< t} := \sum_{j < t} \eta_j\EE\left[f'(\langle w_j,z_j\rangle) | v_0,\ldots, v_{j-1} \right].\end{equation}

	Note that here the expectation is taken with respect to $i_j \sim [n]$. Since $\langle w_j, z_j\rangle = \langle w_j, v_j\rangle + \alpha\|w_j\|$, this is equivalent to taking the expectation over $v_j$, where $v_j$ is drawn uniformly at random from $\{-y_1x_1,\ldots, -y_nx_n\}$. Throughout the following lemmas, we will assume that the step sizes satisfy $\eta_t \leq \min\{1,2(1+\alpha)^{-2}\}$. We first upper bound $\|w_t-w\|$ for any fixed $w$.

	\begin{lemma}\label{lem:log1}
	For any $w \in \real^d$ and $t \geq 1$, 
	$$\|w_t-w\|^2 \leq \|w\|^2 + 2 \sum_{j < t} \eta_j f(\langle w, z_t\rangle).$$
	\end{lemma}

	\begin{proof}
		By the SGD update rule with $\alpha$ adversarial training, we have
		$$\|w_{t+1} -w \|^2 = \|w_t-w\|^2 -2\eta_t \langle g_t, w_t-w\rangle + \eta_t^2\|g_t\|^2.$$
		By \eqref{eq:z3}, $g_t = f'(\langle w_t,z_t\rangle)z_t$, and so by convexity of $f$ we have
		$$-2\eta_t \langle g_t, w_t-w\rangle \leq -2\eta_t\big( f(\langle w_t,z_t\rangle)-f(\langle w,z_t \rangle)\big).$$
		We now wish to bound $\|g_t\|^2$. By direct computation, 
		\begin{align*}
		\|g_t\|^2 &= f'(\langle w_t,z_t\rangle)^2\|z_t\|^2\\
		&\leq f'(\langle w_t,z_t\rangle)\|z_t\|^2\\
		&\leq f(\langle w_t,z_t\rangle)\|z_t\|^2\\
		&\leq f(\langle w_t,z_t\rangle)(1+\alpha)^2.\end{align*}
		The first inequality holds because $0 < f'(a) < 1$ for all $a \in \real$, the second holds by the fact that $f'(a) \leq f(a)$ for all $a$, and the last holds because
		$$\|z_t\| = \left\|y_{i_t}x_{i_t} - \alpha\overline{w}_t\right\| \leq \|x_{i_t}\| + \alpha \leq 1+\alpha.$$
		Therefore, 
		\begin{align*}
		\|w_{t+1}-w\|^2 &\leq \|w_t-w\|^2 -2\eta_t\big( f(\langle w_t,z_t\rangle)-f(\langle w,z_t \rangle)\big) + \eta_t^2f(\langle w_t,z_t\rangle)(1+\alpha)^2\\
		&= \|w_t-w\|^2 + 2\eta_tf(\langle w,z_t\rangle) + \left(\eta_t^2(1+\alpha)^2 - 2\eta_t\right)f(\langle w_t,z_t\rangle)\\
		&\leq \|w_t-w\|^2 + 2\eta_t f(\langle w,z_t\rangle).\end{align*}
		This last step follows from the fact that $\eta_t \leq 2/(1+\alpha)^2$. Recursing on $t$ and using the fact that $w_0 = 0$, we derive the desired result.
	\end{proof}

	Let $u_t = \frac{\ln(t)w^*}{\gamma-\alpha}$. Note that $\|u_t\| = \ln(t)/(\gamma-\alpha)$. By Lemma \ref{lem:log1}, we get the following bound on $\|w_t-u_t\|$.
	\begin{lemma}\label{lem:log2}
	For any $t \geq 1$, 
	$$\|w_t-u_t\|^2 \leq \frac{\ln(t)^2}{(\gamma-\alpha)^2} + 2.$$
	\end{lemma}

	\begin{proof}
		Because $f$ is monotonically increasing and by \eqref{eq:z2}, we have that for all $j$,
		\begin{align*}
		f(\langle u_t,z_j\rangle) &= f\left(\frac{\ln(t)}{\gamma-\alpha}\langle w^*,z_j\rangle\right)\\
		&\leq f\left(\frac{\ln(t)(-\gamma + \alpha)}{\gamma-\alpha}\right)\\
		&= f(-\ln(t))\\
		&\leq \exp(-\ln(t))
		= \frac{1}{t}.\end{align*}

		Here we used the fact that $f(x) \leq \exp(x)$ for all $x$. By Lemma \ref{lem:log1}, we get
		\begin{align*}
		\|w_t-u_t\|^2_2 & \leq \|u_t\|^2 + 2 \sum_{j < t}\eta_j f(\langle u_t, z_j\rangle)\\
		& \leq \dfrac{\ln(t)^2}{(\gamma-\alpha)^2} + 2\sum_{j < t}\frac{\eta_j}{t}\\
		&\leq \dfrac{\ln(t)^2}{(\gamma-\alpha)^2} + 2.
		\end{align*}
	\end{proof}

	We can now use the above lemma to give an upper bound on the sum of derivatives of $f$ up to $t$.

	\begin{lemma}\label{lem:log3}
	For all $t \geq 1$,
	$$f'_{< t} \leq \dfrac{2\ln(t)}{(\gamma-\alpha)^2} + \frac{2}{\gamma-\alpha}.$$
	\end{lemma}

	\begin{proof}
		First, since $w^*$ has margin at least $\gamma$, we have
		\begin{equation}\label{log3_eq1}
		\begin{split}
		\langle w_t-u_t , w^*\rangle &= \langle w_t, w^* \rangle - \langle u_t,w^*\rangle\\
		&= -\sum_{j < t}\eta_j f'(\langle w_j, z_j\rangle)\langle z_j,w^*\rangle - \frac{\ln(t)}{\gamma-\alpha}\\
		&\geq  (\gamma-\alpha) f'_{< t} - \frac{\ln(t)}{\gamma-\alpha}.
		\end{split}
		\end{equation}

		By the Cauchy-Schwarz inequality and Lemma \ref{lem:log2}, we also have
		\begin{equation}\label{log3_eq2}
		\langle w_t-u_t,w^*\rangle \leq \|w_t-u_t\| \leq \sqrt{\frac{\ln(t)^2}{(\gamma-\alpha)^2} + 2}.
		\end{equation}
		Combining \eqref{log3_eq1} and \eqref{log3_eq2} we have
		\begin{align*}
		(\gamma-\alpha) f'_{< t} \leq \frac{\ln(t)}{\gamma-\alpha} + \sqrt{\frac{\ln(t)^2}{(\gamma-\alpha)^2} + 2} \leq \frac{2\ln(t)}{\gamma-\alpha} + 2.
		\end{align*}
	\end{proof}

	In order to bound $F'_{< t}$, we will combine Lemma \ref{lem:log3} with the following martingale Bernstein bound.

	\begin{theorem}[\cite{beygelzimer2011contextual}, Theorem 1]\label{lem:martingale}Let $(X_t)_{t\geq 0}$ be a martingale sequence such that $\EE[X_t] = 0$ and $X_t \leq R$ a.s., and define
	$$S_t := \sum_{j=1}^t X_t$$
	$$V_t := \sum_{j=1}^t \EE [X_t^2].$$
	Then for all $\delta > 0$, with probability at least $1-\delta$,
	$$S_t \leq R\ln(1/\delta) + (e-2)\dfrac{V_t}{R}.$$
	\end{theorem}

	We can now apply the above lemma to $f'_{< t}$ and $F'_{< t}$ to get the following bound.

	\begin{lemma}\label{lem:log4}
		With probability at least $1-\delta$, 
		$$F'_{< t} \leq \frac{8\ln(t)}{(\gamma-\alpha)^2} + \frac{8}{\gamma-\alpha} + 4\ln\left(\frac{1}{\delta}\right).$$
	\end{lemma}

	\begin{proof}
		Recall that $v_t = -y_{i_t}x_{i_t}$ and let $v_{0,t}$ denote the sequence $v_0,\ldots, v_t$. Define

		$$f'_t = \eta_t f'(\langle w_t,z_t\rangle) = \eta_t f'(\langle w_t, v_t\rangle + \alpha\|w_t\|)$$
		$$F'_t = \eta_t\EE [f'(\langle w_t,z_t\rangle) | v_{0,t-1}] = \eta_t\EE_t[f'(\langle w_t,v_t\rangle - \alpha\|w_t\|) | v_{0,t-1}].$$

		Let $X_t = F'_t - f'_t$. Note that $X_t$ is a martingale with respect to the sequence $v_0, v_1, \ldots$ such that $\EE[X_t] = 0$. Since $0 \leq f'(a) \leq 1$ for all $a \in \real$, we have
		$$X_t = \eta_t\big(\EE_t[f'(\langle w_t,z_t) | v_{0,t-1}]-f'(\langle w_t,z_t)\big) \leq \eta_t \leq 1.$$
		Since $\EE [f'_t| v_{0,t-1}] = F'_t$ and $f'_t \leq 1$, we have
		\begin{align*}
		\EE [X_t^2 | v_{0,t-1}] &= (F'_t)^2 -2F'_t\EE[f'_t | v_{0,t-1}] + \EE[(f'_t)^2|v_{0,t-1}]\\
		&= -(F'_t)^2 + \EE[(f'_t)^2 | v_{0,t-1}]\\
		&\leq \EE[(f'_t)^2 | v_{0,t-1}]\\
		&\leq \EE[f'_t | v_{0,t-1}]\\
		&= F'_t.\end{align*}	
		Using Lemma \ref{lem:martingale} and the fact that $\sum_{j=0}^{t-1} X_t = F'_{< t} - f'_{< t}$, we find that with probability at least $1-\delta$,
		$$F'_{<t} - f'_{<t} \leq \ln(1/\delta) + (e-2)F'_{<t}.$$
		Rearranging and applying Lemma \ref{lem:log3}, we have that with probability at least $1-\delta$,
		$$F'_{<t} \leq \dfrac{f'_t + \ln(1/\delta)}{3-e} \leq \dfrac{8\ln(t)}{(\gamma-\alpha)^2} + \frac{8}{\gamma-\alpha} +4\ln(1/\delta).$$
	\end{proof}

	To prove Theorem \ref{thm:log_reg1}, we will need one last auxiliary lemma.

	\begin{lemma}[\cite{ji2018risk}, Lemma 2.6]\label{lem:ji_26}For any $x \in \real$, $f(x) \leq f'(x)(|x|+2)$.\end{lemma}

	\begin{proof}[Proof of Theorem \ref{thm:log_reg1}]
		Recall that $v_t = -y_{i_t}x_{i_t}$. Let $v_{0,t-1}$ denote the sequence $v_0, \ldots, v_{t-1}$. Note that we have
		$$\ell_{\rob}(w_t,x_{i_t},y_{i_t}) = f(\langle w_t,v_t\rangle - \alpha\|w_t\|) = f(\langle w_t,z_t\rangle).$$

		Therefore, we have
		\begin{align*}
		L_{\rob}(w_t) &= \frac{1}{n}\sum_{i=1}^n \ell_{\rob}(w_t,x_i,y_i)\\
		&= \EE_{i \sim [n]} \left[\ell_{\rob}(w_t,x_i,y_i)\right]\\
		&= \EE_{i \sim [n]} \left[f(\langle w_t, -y_ix_i\rangle + \alpha\|w_t\|)\right]\\
		&= \displaystyle\EE_{v \sim \{-y_1x_1,\ldots, -y_nx_n\} }[f(\langle w_t, v\rangle + \alpha\|w_t\|)]\\
		&= \EE \left[f(\langle w_t,v_t\rangle +\alpha\|w_t\|)\big| v_{0,t-1}\right]\\
		&= \EE \left[f(\langle w_t,z_t\rangle)\big| v_{0,t-1}\right].
		\end{align*}
		
		Here we used the fact $w_t = \EE[w_t | v_{0,t-1}]$, as the iterates $w_0,\ldots, w_t$ are fully determined by $v_0, \ldots, v_{t-1}$, since
		\begin{align*}
		w_{j+1} &= w_j - \eta_jf'(\langle w_j, -y_{i_j}x_{i_j}\rangle +\alpha\|w_j\|)\left(-y_{i_j}x_{i_j}+\alpha\frac{w_j}{\|w_j\|}\right)\\
		&= w_j - \eta_jf'(\langle w_j, v_j\rangle +\alpha\|w_j\|)\left(v_j+\alpha\overline{w}_j\right).\end{align*}
		By Lemma \ref{lem:ji_26}, we have
		\begin{equation}\label{eq:final1}
		\begin{split}
		\sum_{j < t}\eta_j L_{\rob}(w_j) &= \sum_{j< t}\eta_j\EE \left[f(\langle w_j,z_j\rangle) \big| v_{0,j-1}\right]\\
		&\leq \sum_{j < t}\eta_j\EE\left[f'(\langle w_j,z_j\rangle)(|\langle w_j,z_j\rangle| + 2) \big| v_{0,j-1}\right]\\
		&\leq \sum_{j < t}\eta_j\EE\left[f'(\langle w_j,z_j\rangle)(\|w_j\|(1+\alpha) + 2) \big| v_{0,j-1}\right].
		\end{split}
		\end{equation}

		This last inequality follows from the fact that $\|z_j\| \leq (1+\alpha)$. By Lemma \ref{lem:log2}, we have that for $j < t$,
		\begin{equation}\label{eq:final2}
		\|w_j\| \leq \|u_j\| + \|w_j-u_j\| \leq \frac{2\ln(j)}{\gamma-\alpha}+2\leq \frac{2\ln(t)}{\gamma-\alpha}+2.\end{equation}

		Combining \eqref{eq:final3} with the fact that $\alpha < \gamma \leq 1$, we then have
		\begin{equation}\label{eq:final3}
		\|w_j\|(1+\alpha) +2 \leq \dfrac{4\ln(t)}{\gamma-\alpha} + 6.\end{equation}

		Combining \eqref{eq:final1} and \eqref{eq:final3}, we have
		\begin{align*}
		\sum_{j < t}\eta_j L_{\rob}(w_j) &\leq \left(\dfrac{4\ln(t)}{\gamma-\alpha}+6\right)\sum_{j < t}\eta_j\EE\left[f'(\langle w_j,z_j\rangle) \big| v_{0,j-1}\right]\\
		&= \left(\dfrac{4\ln(t)}{\gamma-\alpha}+6\right)F'_{< t}.\end{align*}

		Applying Lemma \ref{lem:log4}, this implies that with probability at least $1-\delta$,
		\begin{align*}
		\sum_{j < t}\eta_j L_{\rob}(w_j) &\leq \left(\dfrac{4\ln(t)}{\gamma-\alpha}+6\right)\left(\dfrac{8\ln(t)}{(\gamma-\alpha)^2} + \frac{8}{\gamma-\alpha} +4\ln(1/\delta)\right).\end{align*}

		Let $\hat{w}_t = t^{-1}\sum_{j < t}w_j$. If $\eta_j = \eta$ for all $t$, by Jensen's inequality and using the fact that $L_{\rob}$ is convex (by applying Lemma \ref{lem:linear_robust_loss} to $f$), we have
		\begin{align*}
		L_{\rob}(\hat{w}_t) \leq \frac{1}{t}\sum_{j < t}L_{\rob}(w_j) \leq \frac{1}{\eta t}\left(\dfrac{4\ln(t)}{\gamma-\alpha}+6\right)\left(\dfrac{8\ln(t)}{(\gamma-\alpha)^2} + \frac{8}{\gamma-\alpha} +4\ln(1/\delta)\right).\end{align*}
	\end{proof}

\subsection{Proof of Corollary \ref{thm:log_reg2}}

	\begin{proof}[Proof of Theorem \ref{thm:log_reg2}]
		Given $q > 1$, let $C_q = \inf\{t \geq 3~|~\ln(t)^2 \leq t^{1-1/q}\}$. This is finite by standard properties of the logarithm. Since $\|x\|\leq 1$ for all $(x,y) \in S$, we know $\gamma \leq 1$. Since $t \geq 3 > e$, $\ln(t)/(\gamma-\alpha) \geq 1$. By Theorem \ref{thm:log_reg1} there is some constant $c$ such that with probability at least $1-\delta$,
		\begin{equation}\label{rob_eq1}
		\begin{split}
		L_{\rob}(\hat{w}_t) &\leq \dfrac{c}{\eta t}\left(\dfrac{\ln(t)}{\gamma-\alpha}+1\right)\left(\dfrac{\ln(t)}{(\gamma-\alpha)^2}+\dfrac{1}{\gamma-\alpha}+\ln\left(\dfrac{1}{\delta}\right)\right)\\
		&\leq \dfrac{4c}{\eta t}\left(\dfrac{\ln(t)^2}{(\gamma-\alpha)^3} + \dfrac{\ln(1/\delta)\ln(t)}{\gamma-\alpha}\right)\\
		&\leq \dfrac{4c \ln(t)^2}{\eta t}\left(\dfrac{1}{(\gamma-\alpha)^3} + \dfrac{\ln(1/\delta)}{\gamma-\alpha}\right).		
		\end{split}
		\end{equation}
		Since $t \geq C_q$, \eqref{rob_eq1} implies
		\begin{equation}\label{rob_eq2}
		L_{\rob}(\hat{w}_t) \leq \dfrac{4c}{\eta t^{1/q}}\left(\dfrac{1}{(\gamma-\alpha)^3} + \dfrac{\ln(1/\delta)}{\gamma-\alpha}\right).
		\end{equation}
		Note that by assumption,
		$$t \geq \left[\dfrac{4cn}{\ln(2)\eta}\left(\dfrac{1}{(\gamma-\alpha)^3} + \dfrac{\ln(1/\delta)}{\gamma-\alpha}\right)\right]^{q}.$$
		Combining this with \eqref{rob_eq2}, we find that with probability at least $1-\delta$, $L_{\rob}(\hat{w}_t) \leq \ln(2)/n$. By Lemma \ref{lem:log_reg_margin}, this implies that $\hat{w}_t$ has margin at least $\alpha$ with the same probability.

	\end{proof}

\end{appendix}

\end{document}